\documentclass[preprint, noinfoline,addressatend,imslayout]{imsart}

\usepackage{url}
\RequirePackage[colorlinks,citecolor=blue,urlcolor=blue]{hyperref}
\usepackage[authoryear]{natbib}

\usepackage{amsthm,amsmath,amsfonts,amssymb}
\usepackage{enumerate}
\usepackage{bm,accents}
\usepackage{graphicx}
\usepackage{color}

\startlocaldefs

\newtheorem{thm}{Theorem}%[section]
\newtheorem{cor}[thm]{Corollary}
\newtheorem{lem}[thm]{Lemma}

\newtheorem{defn}[thm]{Definition}
\newtheorem{rem}[thm]{Remark}
\newtheorem{example}[thm]{Example}

\setattribute{tablecaption}{shape}{}

\setattribute{tablename} {skip}{.~}

\newcommand\clip[1]{\accentset{\frown}{#1}}
\newcommand{\norm}[1]{\left\Vert#1\right\Vert}

\newcommand{\indi}[1]{\bm{1}{\left\{#1\right\}}}

\newcommand{\R}{\mathbb{R}}
\newcommand{\eps}{\varepsilon}

\newcommand{\argmin}{\operatornamewithlimits{argmin}}
\newcommand{\asarrow}[0]{\stackrel{\mbox{\scriptsize\rm a.s.}}
{\rightarrow}}
\newcommand{\parrow}[0]{\stackrel{\mbox{\scriptsize\rm P}}
{\rightarrow}}

\newcommand{\Y}[0]{\mathcal{Y}}
\newcommand{\T}[0]{\mathcal{T}}
\newcommand{\Z}[0]{\mathcal{Z}}

\newcommand{\Risk}[0]{\mathcal{R}}
\newcommand{\RLP}[0]{\mathcal{R}_{L,P}}
\newcommand{\RLPS}[0]{\mathcal{R}_{L,P}^*}

\newcommand{\RLD}[0]{\mathcal{R}_{L,D}}

\newcommand{\RLND}[0]{\mathcal{R}_{L^n_D,D}}
\newcommand{\RLNG}[0]{\mathcal{R}_{L_G,D}}
\newcommand{\RLG}[0]{\mathcal{R}_{L_G,P}}
\newcommand{\sn}[0]{\sqrt{n}}
\newcommand{\bbeta}[0]{{\bm{\beta}}}

\newcommand{\Ghat}[0]{\hat{G}}

\newcommand{\Lc}{\mathfrak{L}^c}
\newcommand{\nhs}[1]{\|#1\|_H^2}

\newcommand{\sign}{\mathrm{sign}}
\newcommand{\ep}{\mathbb{P}}
\newcommand{\fc}{f^c_{D,\lambda}}
\newcommand{\fp}{f_{P,\lambda}}
\newcommand{\fps}{f_{P,L}^*}

\newcommand{\hatG}{\hat{G}_n}

\newcommand{\clipfc}{\clip{f}^c_{D,\lambda}}

\newcommand{\tV}{\tilde{V}}

\bibpunct{(}{)}{;}{a}{,}{,}
\def\spacingset#1{\renewcommand{\baselinestretch}%
{#1}\small\normalsize} \spacingset{1.6}

\arxiv{1202.5130}

\begin{document}

\begin{frontmatter}

% "Title of the paper"

\title{Support Vector Regression for Right Censored Data\protect\thanksref{T1}}
\runtitle{SVM for Right Censored Data}
\thankstext{T1}{The authors are grateful to anonymous associate editor for the helpful suggestions and comments. The authors thank Danyu Lin for many helpful discussions and suggestions. The first author was funded in part by a Gillings Innovation Laboratory (GIL) award at the UNC Gillings School of Global Public Health. The second author was funded in part by NCI grant CA142538.}

\begin{aug}
\author{\fnms{Yair} \snm{Goldberg}\ead[label=e1]{ygoldberg@stat.haifa.ac.il}}
and \author{\fnms{Michael R.}
\snm{Kosorok}\ead[label=e2]{kosorok@unc.edu}}

\affiliation{University of Haifa and The University of North Carolina at Chapel Hill}

\address{Yair Goldberg\\Department of Statistics
\\The University of Haifa\\ Mount Carmel, Haifa 31905, Israel\\ \printead{e1}}

\address{Michael R. Kosorok\\Department of Biostatistics
\\The University of North Carolina at Chapel
Hill\\Chapel Hill, NC 27599, U.S.A.\\ \printead{e2}}

\runauthor{Goldberg and Kosorok}
\end{aug}
\begin{abstract}
We develop a unified approach for classification and regression support vector machines for data subject to right censoring.
We provide finite sample bounds on the generalization error of the algorithm, prove risk consistency for a wide class of probability measures, and study the associated learning rates. We apply the general methodology to estimation of the (truncated) mean, median, quantiles, and for classification problems. We present a simulation study that demonstrates the performance of the proposed approach.
\end{abstract}

\begin{keyword}[class=AMS]
\kwd[Primary ]{62G05}
\kwd[; secondary ]{ 62G20}\kwd{62N02}
\end{keyword}

\begin{keyword}
\kwd{support vector regression}
\kwd{right censored data}
\kwd{generalization error}
\kwd{universal consistency}
\end{keyword}
\end{frontmatter}

\section{Introduction}\label{sec:intro}
In many medical studies, estimating the failure time
distribution function, or quantities that depend on this distribution, as a function of patient demographic and prognostic variables, is of central importance for risk assessment and health planing. Frequently, such data is subject to right censoring. The goal of this paper is to develop tools for analyzing such data using machine learning techniques.

Traditional approaches to right censored failure time analysis include using parametric models, such as the Weibull distribution, and semiparametric models such as proportional hazard models
\citep[see][for both]{Lawless03}. Even when less stringent models---such as nonparametric estimation---are
used, it is typically assumed that the distribution
function is smooth in both time and covariates
\citep{Dabrowska87,Gonzalez94}. These assumptions seem
restrictive, especially when considering today's high-dimensional data settings.
% a novel Q-learning algorithm that
%takes into account the censored nature of the observations
%using inverse-probability-of-censoring weighting (see,
%\citealp{Robins1994erc}

In this paper, we propose a support vector machine (SVM)
learning method for right censored data. The choice of SVM is motivated by the fact that SVM learning methods are easy-to-compute techniques that enable estimation under weak or no assumptions on the
distribution \citep{SVR}. SVM learning methods, which we review in detail in Section~\ref{sec:notation}, are a collection of algorithms that attempt to minimize the risk with respect to some loss function. An SVM
learning method typically minimizes a regularized version
of the empirical risk over some reproducing kernel Hilbert
space (RKHS). The resulting minimizer is referred to as the SVM decision function. The SVM learning method is the mapping that assigns to each data set its corresponding SVM decision function.

We adapt the SVM framework to
right censored data as follows. First, we represent the distribution's quantity of interest as a Bayes decision function, i.e., a function
that minimizes the risk with respect to a loss function.
We then construct a data-dependent version of this loss
function using inverse-probability-of-censoring weighting
\citep{Robins94}. We then minimize a regularized empirical risk
with respect to this data-dependent loss function to obtain
an SVM decision function for censored data. Finally, we
define the SVM learning method for censored data as the mapping
that assigns for every censored data set its corresponding SVM decision function.

Note that unlike the standard SVM decision function, the proposed censored SVM decision function is obtained as the minimizer of a
data-dependent loss function. In other words, for each data set, a different minimization loss function is defined.
Moreover, minimizing the empirical risk no longer consists of minimizing a sum of i.i.d. observations.
Consequently, different techniques are needed to study the theoretical properties of the censored SVM learning method.

We prove a number of theoretical results for the proposed
censored SVM learning method. We first prove that the
censored SVM decision function is measurable and unique. We
then show that the censored SVM learning method is a
measurable learning method. We provide a probabilistic
finite-sample bound on the difference in risk between the
learned censored SVM decision function and the Bayes risk. We further show that the SVM learning
method is consistent for every probability measure for
which the censoring is independent of the failure time
given the covariates, and the probability that no censoring
occurs is positive given the covariates. Finally, we
compute learning rates for the censored SVM learning
method. We also provide a simulation study that demonstrates the
performance of the proposed censored SVM learning method. Our results are obtained under some conditions
on the approximation RKHS and the loss function, which can
be easily verified. We also assume that the estimation of
censoring probability at the observed points is consistent.

We note that a number of other learning algorithms have been suggested for survival data. \citet{NN98} and \citet{Ripley} used neural networks. \citet{Segal1988}, \citet{Hothorn2004Bagging}, \citet{Ishwaran2008Random}, and \citet{Zhu2011Recursively}, among others, suggested versions
of splitting trees and random forests for survival data. \citet{Johnson04}, \citet{SVR07}, \citet{SVQR09},
and \citet{Zhao10}, among others, suggested versions
of SVM different from the proposed censored SVM. The theoretical properties of most of these algorithms have never been studied. Exceptions include the consistency proof of \citet{Ishwaran2010} for random survival trees, which requires the assumption that the feature space is discrete and finite. In the context of multistage decision problems, \citet{GK_CQL11} proposed a Q-learning algorithm for right censored data for which a theoretical justification is given, under the assumption that the censoring is independent of both failure time and covariates. However, both of these theoretically justified algorithms are not SVM learning methods. Therefore, we believe that the proposed censored SVM and the accompanying theoretical evaluation given in this paper represent a significant innovation in developing methodology for learning in survival data.

Although the proposed censored SVM approach enables the
application of the full SVM framework to right censored data,
one potential drawback is the need to
estimate the censoring probability at observed failure times.
This estimation is required in order to use
inverse-probability-of-censoring weighting for constructing
the data-dependent loss function. We remark that in many
applications it is reasonable to assume that the censoring
mechanism is simpler than the failure-time distribution; in these cases, estimation of the censoring distribution is typically easier than estimation of the failure distribution.
For example, the censoring may depend only on a subset of
the covariates, or may be independent of the covariates;
in the latter case, an efficient estimator exists.
Moreover, when the only source of censoring is administrative, in other words, when the data is censored because the study ends at a prespecified time, the censoring distribution is often known to be independent of the covariates. Fortunately, the results presented in
this paper hold for any censoring estimation
technique. We present results for both correctly specified
and misspecified censoring models. We also discuss in detail the special cases of the Kaplan-Meier and the Cox
model estimators \citep{FH}.

While the main contribution of this paper is the proposed censored SVM learning method and the study of its properties, an additional contribution is the development of
a general machine learning framework for right censored data. The
principles and definitions that we discuss in the context of
right censored data, such as learning methods,
measurability, consistency, and learning rates, are
independent of the proposed SVM learning method. This framework can be adapted to other
learning methods for right censored data, as well as for
learning methods for other missing data mechanisms.

The paper is organized as follows. In
Section~\ref{sec:notation} we review right-censored data
and SVM learning methods. In Section~\ref{sec:no_censoring}
we briefly discuss the use of SVM for right-censored data when no
censoring is present. Section~\ref{sec:censoring} discusses
the difficulties that arise when applying SVM to right
censored data and presents the proposed censored SVM learning
method. Section~\ref{sec:main} contains the main theoretical
results, including finite sample bounds and consistency.
Simulations appear in Section~\ref{sec:simulation}.
Concluding remarks appear in Section~\ref{sec:summary}. The lengthier
key proofs are provided in the Appendix. Finally, the Matlab code for both the algorithm and the simulations can be found in \ref{sec:suppA}.

\section{Preliminaries}\label{sec:notation}
\defcitealias{FH}{FH91}
\defcitealias{SVR}{SC08}
In this section, we establish the notation used throughout the paper.
We begin by describing the data setup (Section~\ref{subsec:right_censored_data}). We then discuss
loss functions (Section~\ref{subsec:loss}). Finally we discuss SVM learning methods (Section~\ref{subsec:svm}). The notation
for right censored data generally follows \citet{FH} \citepalias[hereafter
abbreviated][]{FH}. For the loss function and the SVM definitions, we follow
\citet{SVR} \citepalias[hereafter abbreviated][]{SVR}.

\subsection{Data Setup}\label{subsec:right_censored_data}
We assume the data consist of $n$ independent and identically-distributed random triplets
$D=\{(Z_1,U_1,\delta_1),\ldots,(Z_n,U_n,\delta_n)\}$. The
random vector $Z$ is a covariate vector that takes its
values in a set $\Z\subset\R^d$. The random variable $U$ is the observed time defined by $U=T\wedge C$, where $T\geq 0$ is the failure time, $C$ is the censoring time, and where $a\wedge
b=\min(a,b)$. The indicator $\delta=\indi{T\leq C}$ is the failure indicator, where $\indi{A}$ is $1$ if $A$ is true and $0$ otherwise, i.e., $\delta=1$ whenever a failure time
is observed.

Let $S(t|Z)=P(T>t|Z)$ be the survival functions of $T$, and let
$G(t|Z)=P(C> t|Z)$ be the survival function of $C$. We make the following assumptions:
\begin{enumerate}
\renewcommand{\labelenumi}{(A\arabic{enumi})}
  \item $C$ takes its values in the segment $[0,\tau]$ for some finite $\tau>0$,
      and $\inf_{z\in\Z} G(\tau-|z)\geq
      2K>0$.\label{as:positiveRisk}
  \item $C$ is independent of $T$, given
      $Z$.\label{as:T_independent_C}
%  \item The probability of simultaneous failure and censoring is zero.\label{as:no_simultaneous_failure}
\end{enumerate}
The first assumption assures that there is a positive probability of censoring over the observation time
range ($[0,\tau]$). Note that the existence of such a $\tau$  is typical since most studies have
a finite time period of observation. In the above, we also define $F(t-)$ to be the left-hand limit of
a right continuous function $F$ with left-hand limits. The second assumption is standard in survival analysis and ensures
that the joint nonparametric distribution of the survival and censoring times, given the covariates, is
identifiable.
%While this assumption may be relaxed, this assumptions
%is made to avoid complications that are not of interest to
%the main results of this paper. For discussions on these
%assumptions, see, for example, CITE??,
%While these two assumptions can be relaxed, these assumptions are
%made to avoid complications that are not of interest to the main
%results of this paper. For discussions on these assumptions, see,
%for example, CITE, for the first assumption and \citep{Satten01} for
%the second assumption.

We assume that the censoring mechanism can be described by
some simple model. Below, we consider two possible
examples, although the main results do not require any
specific model. First, we need some notation. For every
$t\in[0,\tau]$, define $\mathbf{N}(t)=\indi{U\leq t,\delta=0}$ and
$\mathbf{Y}(t)=\indi{U>t}+\indi{U=t,\delta=0}$. Note that
since we are interested in the survival function of the
censoring variable, $\mathbf{N}(t)$ is the counting process for the
censoring, and not for the failure events, and
$\mathbf{Y}(t)$ is the at-risk process for observing a
censoring time. For a cadlag function $A$ on $(0,\tau]$,
define the product integral $\phi(A)(t)=\prod_{0<s\leq
t}(1+dA(s))$ \citep{VW96}. Define $\ep_n$ to be the
empirical measure, i.e., $\ep_n f(X)=n^{-1}\sum_{i=1}^n
f(X_i)$. Define $Pf$ to be the expectation of $f$ with respect to $P$.

\begin{example}\label{ex:KM}
\textbf{Independent censoring:} Assume that $C$ is independent of
both $T$ and $Z$. Define
\begin{align*}
    \hat\Lambda(t)=\int_0^t \frac{\ep_n d\mathbf{N}(s)}{\ep_n \mathbf{Y}(s)}\,.
\end{align*}
Then $\hatG (t)=\phi(-\hat\Lambda)(t)$ is the Kaplan-Meier estimator
for $G$. $\hatG $ is a consistent and efficient estimator for
the survival function $G$ \citepalias{FH}.
\end{example}
\begin{example}\label{ex:ph}
\textbf{The proportional hazards model:} Consider the case
that the hazard of $C$ give $Z$ is of the form
$e^{Z'\bbeta}d\Lambda$ for some unknown vector
$\bbeta\in\R^d$ and some continuous unknown nondecreasing
function $\Lambda$ with $\Lambda(0)=0$ and $0 <
\Lambda(\tau)<\infty$. Let $\hat{\bbeta}$ be the zero of
the estimating equation
\begin{align*}
  \Phi_n(\bbeta) = \ep_n\int_0^\tau \left(Z-\frac{\ep_n Z\mathbf{Y}(s)e^{\bbeta'Z}}{\ep_n
  \mathbf{Y}(s)e^{\bbeta'Z}}\right)d\mathbf{N}(s)\,.
\end{align*}
Define
\begin{align*}
    \hat\Lambda(t)=\int_0^t \frac{\ep_n d\mathbf{N}(s)}{\ep_n
    \mathbf{Y}(s)e^{\hat\bbeta'Z}}\,.
\end{align*}
Then $\hatG (t|z)=\phi(-e^{\hat\bbeta'z}\hat\Lambda(t))$ is a
consistent and efficient estimator for survival function $G$
\citepalias{FH}.
\end{example}

Even when no simple form for the censoring mechanism is assumed, the censoring distribution can be estimated using a generalization of the Kaplan-Meier estimator of Example~\ref{ex:KM}.
\begin{example}\label{ex:Generalized_KM}
\textbf{Generalized Kaplan-Meier:} Let $k_{\sigma}:\Z\times Z\mapsto\R$ be a kernel function of width $\sigma$. Define $\hat{\mathbf{N}}(t,z)=K(z,Z)\indi{U\leq t,\delta=0}$ and $\hat{\mathbf{Y}}(t,z)=K(z,Z)(\indi{U>t}+\indi{U=t,\delta=0})$. Define
\begin{align*}
    \hat\Lambda(t|z)=\int_0^t \frac{\ep_n d\hat{\mathbf{N}}(s,z)}{\ep_n \mathbf{Y}(s,z)}\,.
\end{align*}
Then the generalized Kaplan-Meier estimator is given by $\hatG (t|z)=\phi(-\hat\Lambda)(t|z)$, where the product integral $\phi$ is defined for every fixed $z$. Under some conditions, \citet{Dabrowska87,Dabrowska89} proved consistency of the estimator and discussed its convergence rates.
\end{example}

Usually we denote the estimator of the survival function of the censoring variable $G(t|Z)$ by $\hatG  (t|Z)$ without referring to a specific estimation method. When needed, the specific estimation method will be discussed. When independent censoring is assumed, as in Example~\ref{ex:KM}, we denote the estimator by $\hatG (t)$.

\begin{rem}\label{rem:survival_greater than_zero}
By Assumption (A\ref{as:positiveRisk}), $\inf_{z\in \Z}
G(\tau|z)\geq 2K>0$, and thus if the estimator $\hatG $ is
consistent for $G$, then, for all $n$ large enough, $\inf_{z\in \Z}\hatG (\tau|z )>K>0$. In the following, for simplicity, we assume that the estimator
$\hatG $ is such that $\inf_Z \hatG (\tau|Z )>K>0$. In general, one can always replace $\hatG $ by $\hatG  \vee  K_n$, where $K_n\rightarrow 0$.
In this case, for all $n$ large enough, $\inf_Z \hatG (\tau|Z )>K>0$ and for all $n$, $\inf\hatG>0$.

%This can be
%done, even if $\inf_{\{z\in \Z\}} G(\tau|z)$ is not bounded from
%below. For example, one can replace $\hatG $ by $\hatG  \wedge K_n$.
\end{rem}

\subsection{Loss Functions}\label{subsec:loss} Let the input space $(\Z,\mathcal A)$ be a measurable space.
Let the response space $\Y$ be a closed subset of $\R$. Let
$P$ be a measure on $\Z\times\Y$.

A function $L:\Z\times\Y\times\R\mapsto [0,\infty)$ is
a \emph{loss function} if it is measurable. We say
that a loss function $L$ is \emph{convex} if $L(z,y,\cdot)$
is convex for every $z\in\Z$ and $y\in\Y$. We say that a
loss function $L$ is \emph{locally Lipschitz continuous}
with Lipschitz local constant function $c_L(\cdot)$ if for every $a>0$
\begin{align*}
    \sup_{\substack{z\in\Z\\y\in\Y}}|L(z,y,s)-L(z,y,s')|<c_L(a) |s-s'| \,,\quad s,s'\in [-a,a]\,.
\end{align*}
We say that $L$ is \emph{Lipschitz continuous} if there is a
constant $c_L$ such that the above holds for any $a$ with
$c_L(a)=c_L$.

For any measurable function $f:\Z\mapsto\R$ we define the
\emph{$L$-risk} of $f$ with respect to the measure $P$ as $\Risk_{L,P}(f)=E_{P}[L(Z,Y,f(Z))]$. We define the
\emph{Bayes risk} $\Risk^*_{L,P}$ of $f$ with respect to loss function $L$ and measure
$P$ as $\inf_f \Risk_{L,P}(f)$, where the infimum is taken
over all measurable functions $f:\Z\mapsto\R$. A function
$f^*_{L,P}$ that achieves this infimum is called a Bayes
decision function.

%In many cases, the function to be estimated are Bayes
%decision function for some specific loss function. Here are
%a few examples.
%
%\begin{example}\label{ex:classification}
%\textbf{Binary classification:} Assume that $Y=\{-1,1\}$.
%We would like to find a function $f:\Z\mapsto\{-1,1\}$ that
%such that for almost every $x$ $P(f(x)=Y|X=x)\geq 1/2$. One
%can think of $f$ as a function that predicate the label $y$
%of a pair $(z,y)$ when only $x$ is observed. In this case,
%the desired function is the Bayes decision function
%$f^*_{L,P}$ for the loss function $L_{\mathrm{BC}}(z,y,t)=
%\indi{y\cdot\sign t\neq 1}$.\end{example}
%
%\begin{example}\label{ex:expectation}
%\textbf{Expectation:} Assume that $\Y=\R$. We would like
%find estimate the expectation of $Y|X$. The conditional
%expectation function is the Bayes decision function
%$f^*_{L,P}$ for the square loss function
%$L_{\mathrm{LS}}(z,y,t)= (y-t)^2$.
%\end{example}
%
%\begin{example}\label{ex:median}
%\textbf{Median and Quntiles:} Assume that $\Y=\R$. We would
%like find estimate median of $Y|X$. The conditional median
%is the Bayes decision function $f^*_{L,P}$ for the
%absolute value loss function $L_{\mathrm{AD}}(z,y,t)=
%|y-t|$. Similarly, the $\alpha$-quantile of $Y|X$ is
%obtained by the minimizer of the loss function
%\begin{align*}
%    L_{\alpha}(z,y,t)=\left\{
%               \begin{array}{ll}
%                 -(1-\alpha)(y-t)\quad & \mathrm{if\,} t\geq y \\
%                 \alpha(y-t) \quad& \mathrm{if\,}t<y \\
%               \end{array}
%             \right.
%\end{align*}
%
%\end{example}
We now present a few examples of loss functions and their
respective Bayes decision functions. In the next section we discuss the use of these loss functions for right censored data.

\begin{example}\label{ex:classification}
\textbf{Binary classification:} Assume that $\Y=\{-1,1\}$.
We would like to find a function $f:\Z\mapsto\{-1,1\}$ such that for almost every $z$, $P(f(z)=Y|Z=z)\geq 1/2$.
One can think of $f$ as a function that predicts the label
$y$ of a pair $(z,y)$ when only $z$ is observed. In this
case, the desired function is the Bayes decision function
$f^*_{L,P}$ with respect to the loss function $L_{\mathrm{BC}}(z,y,s)=
\indi{y\cdot\sign (s)\neq 1}$. In practice, since the
loss function $L_{\mathrm{BC}}$ is not convex, it is usually replaced by
the hinge loss function
$L_{\mathrm{HL}}(z,y,s)=\max\{0,1-ys\}$.
\end{example}

\begin{example}\label{ex:expectation}
\textbf{Expectation:} Assume that $\Y=\R$. We would like
to estimate the expectation of the response $Y$ given the covariates $Z$. The conditional
expectation is the Bayes decision function
$f^*_{L,P}$ with respect to the squared error loss function
$L_{\mathrm{LS}}(z,y,s)= (y-s)^2$.
\end{example}

\begin{example}\label{ex:median}
\textbf{Median and quantiles:} Assume that $\Y=\R$. We would
like to estimate the median of $Y|Z$. The conditional median
is the Bayes decision function $f^*_{L,P}$ for the
absolute deviation loss function $L_{\mathrm{AD}}(z,y,s)=
|y-s|$. Similarly, the $\alpha$-quantile of $Y$ given $Z$ is
obtained as the Bayes decision function for the loss function
\begin{align*}
    L_{\alpha}(z,y,s)=\left\{
               \begin{array}{ll}
                 -(1-\alpha)(y-s)\quad & \mathrm{if\,} s\geq y \\
                 \alpha(y-s) \quad& \mathrm{if\,}s<y \\
               \end{array}
             \right.\quad ,\;\alpha\in(0,1)\,.
\end{align*}

\end{example}

Note that the functions $L_{\mathrm{HL}}$,
$L_{\mathrm{LS}}$, $L_{\mathrm{AD}}$, and $L_{\alpha}$ for
$\alpha\in(0,1)$ are all convex. Moreover, all these
functions except $L_{\mathrm{LS}}$ are Lipschitz continuous,
and $L_{\mathrm{LS}}$ is locally Lipschitz continuous when
$\Y$ is compact.

\subsection{Support Vector Machine (SVM) Learning Methods}\label{subsec:svm}

Let $L$ be a convex locally Lipschitz continuous loss
function. Let $H$ be a separable reproducing kernel Hilbert
space (RKHS) of a bounded measurable kernel on $\Z$ \citepalias[for details regarding RKHS, the reader is referred to][Chapter~4]{SVR}. %Fix
%$\lambda>0$. We say that $f_{P,\lambda,H}$ is a
%\emph{general SVM decision function} if $
%f_{P,\lambda,H}\in H$ and
%\begin{align*}
%f_{P,\lambda,H}=\argmin_{f\in H}\lambda\|f\|_H^2+\Risk_{L,P}(f)\,,
%\end{align*}
%where $\|f\|_H$ is the norm of $f$ in the Hilbert space
%$H$. %It can be shown that $f_{P,\lambda,H}$ exists and is
%unique \citepalias[][Corollary 5.3]{SVR}.

Let $D_0=\{(Z_1,Y_1),\ldots,(Z_n,Y_n)\}$ be a set of $n$ i.i.d.
observations drawn according to the probability measure
$P$. Fix $\lambda$ and let $H$ be as above. Define the
\emph{empirical SVM decision function}
\begin{align}\label{eq:svm_decision_function}
    f_{D_0,\lambda}=\argmin_{f\in H}\lambda \|f\|_H^2+\Risk_{L,D_0}(f)\,,
\end{align}
where
\begin{align*}
\Risk_{L,D_0}(f)\equiv\ep_n
L(Z,Y,f(Z))\equiv \frac{1}{n}\sum_{i=1}^n L(Z_i,Y_i,f(Z_i))
\end{align*}
is the empirical risk.

For some sequence $\{\lambda_n\}$, define the \emph{SVM learning method $\mathfrak{L}$}, as the map
\begin{align}\label{eq:svm_learning_method}
\begin{split}
(\Z\times\Y)^n\times \Z&\mapsto \R\\
 (D_0,z)&\mapsto f_{D_0,\lambda_n}
 \end{split}
\end{align}
for all $n\geq 1$. We say that $\mathfrak{L}$ is
\emph{measurable} if it is measurable for all $n$ with
respect to the minimal completion of the product
$\sigma$-field on $(\Z\times\Y)^n\times\Z$. We say that
that $\mathfrak{L}$ is ($L$-risk) $P$-consistent if for all $\eps>0$
\begin{align}\label{eq:universal_consistency}
    \lim_{n\rightarrow\infty} P(D_0\in (Z\times\Y)^n \,:\,\Risk_{L,P}(f_{D_0,\lambda_n})\leq
    \Risk_{L,P}^*+\eps)=1\,.
\end{align}
We say that $\mathfrak{L}$ is \emph{universally consistent} if for all distributions $P$ on $\Z\times\Y$, $\mathfrak{L}$ is $P$-consistent.

%Note that this include
%the functions the loss functions,$L_{\mathrm{LS}}, L_{\mathrm{AD}}$,
%$L_{\alpha}$, and $L_{\mathrm{BC}}$.
We now briefly summarize some known results regarding SVM learning methods needed for our exposition. More advanced results can be obtained using conditions on the functional spaces and clipping. We will discuss these ideas in the context of censoring in Section~\ref{sec:main}.
\begin{thm}\label{thm:regular_SVM}
Let $L:\Z\times\Y\times\R\mapsto [0,\infty)$ be a convex Lipschitz continuous loss function such that $L(z,y,0)$ is uniformly bounded. Let $H$ be a separable RKHS of a bounded measurable kernel on the set
$\Z\subset\R^d$. Choose
$0<\lambda_n<1$ such that $\lambda_n\rightarrow 0$, and
$\lambda_n^2 n\rightarrow \infty$. Then
\begin{enumerate}[(a)]
  \item The empirical SVM decision function $f_{D_0,\lambda_n}$ exists and is unique.
  \item The SVM learning method $\mathfrak{L}$ defined
      in~\eqref{eq:svm_learning_method} is measurable.
  \item The $L$-risk $\Risk_{L,P}(f_{D_0,\lambda_n})\parrow \inf_{f\in
H}\Risk_{L,P}(f)$.
\item If the RKHS $H$ is dense in the set of integrable functions on $\Z$, then the SVM learning method $\mathfrak{L}$ is universally
consistent.
\end{enumerate}
\end{thm}
The proof of~(a) follows from \citetalias{SVR},
Lemma~5.1 and Theorem~5.2. For the proof of~(b), see
\citetalias{SVR}, Lemma~6.23. The proof of~(c)
follows from  \citetalias{SVR} Theorem~6.24. The proof
of~(d) follows from \citetalias{SVR}, Theorem~5.31, together with Theorem~6.24.

\section{SVM for Survival Data without Censoring}\label{sec:no_censoring}
In this section we present a few examples of the use of SVM
for survival data but without censoring. We show how different quantities obtained from the conditional distribution of $T$ given $Z$ can be represented as Bayes
decision functions. We then show how SVM learning methods
can be applied to these estimation problems and briefly review
theoretical properties of such SVM learning methods. In
the next section we will explain why these standard SVM
techniques cannot be employed directly when censoring is
present.

Let $(Z,T)$ be a random vector where $Z$ is a covariate
vector that takes its values in a set
$\Z\subset\R^d$, $T$ is survival time that takes it values
in $\T=[0,\tau]$ for some positive constant $\tau$, and
where $(Z,T)$ is distributed according to a probability
measure $P$ on $\Z\times\T$.

Note that the conditional expectation $P[T|Z]$ is the
Bayes decision function for the least squares loss function
$L_{\mathrm{LS}}$. In other words
\begin{align*}
P[T|Z]=\argmin_f P[L_{\mathrm{LS}}(Z,T,f(Z))]\,,
\end{align*}
where the minimization is taken over all measurable
real functions on $\Z$ (see Example~\ref{ex:expectation}).
Similarly, the conditional median and the $\alpha$-quantile
of $T|Z$ can be shown to be the Bayes decision functions
for the absolute deviation function $L_{\mathrm{AD}}$ and
$L_{\alpha}$, respectively (see Example~\ref{ex:median}).
In the same manner, one can represent other quantities of
the conditional distribution $T|Z$ using Bayes decision
functions.

Defining quantities computed from the survival function as Bayes
decision functions is not limited to regression (i.e., to
a continuous response). Classification problems can also
arise in the analysis of survival data \citep[see, for
example,][]{Ripley,Johnson04}. For example, let $\rho$,
$0<\rho<\tau$, be a cutoff constant. Assume that survival to a
time greater than $\rho$ is considered as death
unrelated to the disease (i.e., remission) and a survival
time less than or equal to $\rho$ is considered as death
resulting from the disease. Denote
\begin{align}\label{eq:Y_as_cutoff}
    Y(T)=\left\{
        \begin{array}{cc}
          1& T>\rho \\
          -1 & T\leq \rho \\
        \end{array}
      \right.\,.
\end{align}
In this case, the decision function that predicts remission
when the probability of $Y=1$ given the covariates is
greater than $1/2$ and failure otherwise is a Bayes
decision function for the binary classification loss
$L_{\mathrm{BC}}$ of Example~\ref{ex:classification}.

Let $D_0=\{(Z_1,T_1),\ldots,(Z_n,T_n)\}$ be a data set of $n$
i.i.d. observations distributed according to $P$. Let
$Y_i=Y(T_i)$ where $Y(\cdot):\T\mapsto \Y$ is some
deterministic measurable function. For regression problems,
$Y$ is typically the identity function and for
classification $Y$ can be defined, for example, as
in~\eqref{eq:Y_as_cutoff}. Let $L$ be a convex locally
Lipschitz continuous loss function,
$L:\Z\times\Y\times\R\mapsto [0,\infty)$. Note that this
includes the loss functions $L_{\mathrm{LS}},
L_{\mathrm{AD}}$, $L_{\alpha}$, and $L_{\mathrm{HL}}$.
Define the empirical decision function as
in~\eqref{eq:svm_decision_function} and the SVM learning
method $\mathfrak{L}$ as in~\eqref{eq:svm_learning_method}.
Then it follows from Theorem~\ref{thm:regular_SVM} that for an
appropriate RKHS $H$ and regularization sequence
$\{\lambda_n\}$, $\mathfrak{L}$ is measurable and
universally consistent.

\section{Censored SVM}\label{sec:censoring}
In the previous section, we presented a few examples of the use of SVM for survival data without censoring. In this section we explain why standard SVM techniques cannot be applied directly when censoring is present. We then explain how to use inverse probability of censoring weighting \citep{Robins94} to obtain a censored SVM learning method. Finally, we show that the obtained censored SVM learning method is well defined.

Let $D=\{(Z_1,U_1,\delta_1),\ldots,(Z_n,U_n,\delta_n)\}$ be
a set of $n$ i.i.d. random triplets of right censored data
(as described in Section~\ref{subsec:right_censored_data}). Let $L:Z\times\Y\times\R\mapsto[0,\infty)$ be
a convex locally Lipschitz loss function. Let $H$ be a separable RKHS of a bounded measurable kernel on $\Z$. We
would like to find an empirical SVM decision function. In
other words, we would like to find the minimizer of
\begin{align}\label{eq:find_minimizer_svm_decision_function}
    \lambda \|f\|_H^2 +\RLD(f)\equiv    \lambda \|f\|_H^2 +\frac 1n \sum_{i=1}^n L(Z_i,Y(T_i),f(Z_i))
\end{align}
where $\lambda>0$ is a fixed constant, and $Y:\T\mapsto\Y$ is a known function. The problem is that the
failure times $T_i$ may be censored, and thus unknown. While a simple solution is to ignore the censored observations, it is well known that this can lead to severe bias \citep{Tsiatis2006Semiparametric}.

In order to avoid this bias, one can reweight the
uncensored observations. Note that at time $T_i$, the
$i$-th observation has probability $G(T_i-|Z_i)\equiv
P(C_i\geq T_i |Z_i)$ not to be censored, and thus, one can
use the inverse of the censoring probability for
reweighting in
\eqref{eq:find_minimizer_svm_decision_function}
\citep{Robins94}. %
%\begin{align*}
%    E\left[\left.\frac{\delta}{G(T|Z)}\right| Z,T\right]=1\,.
%\end{align*}
%Hence,
%\begin{align}\label{eq:icpw}
%E\left[L(Z,Y,f(Z))\right]
%&=E\left[L(Z,Y,f(Z))E\left[\left.\frac{\delta}{G(T|Z)}\right| T,Z\right]\right]\\
%&=E\left[E\left[\left. L(Z,Y,f(Z))\frac{\delta}{G(T|Z)}\right| T,Z\right]\right]\nonumber\\
%&=E\left[L(Z,Y,f(Z))\frac{\delta}{G(T|Z)}\right]\,.\nonumber
% \end{align}
%Thus, the minimizer of $\lambda \|f\|_H^2 +E [L(Z,Y,f(Z))]$
%equals to the minimizer of $\lambda \|f\|_H^2 +E[
%L(Z,Y,f(Z))\delta/G(T|Z)]$.

More specifically, define the random loss function $L^n: (\Z\times
\T\times\{0,1\})^n\times(\Z\times \T\times\{0,1\}\times \R) \mapsto
\R$ by
\begin{align*}
    L^n(D,(z,u,\delta,s))= \left\{
                   \begin{array}{cc}
                     \frac{L(z,Y(u),s)}{\hatG (u|z)}, & \delta=1, \\
                     0, & \delta=0, \\
                   \end{array}
                 \right.
\end{align*}
where $\hatG $ is the estimator of the survival function of the censoring variable based on the set of $n$ random triplets $D$ (see Section~\ref{subsec:right_censored_data}). When $D$ is given, we denote $L^n_D(\cdot)\equiv L^n(D,\cdot)$. Note that in this case the function $L^n_D$ is no longer random. In order to show that $L^n_D$ is a loss function, we need to show that $L^n_D$ is a measurable function.
\begin{lem}\label{lem:Lnd_is_measurable}
Let $L$ be a convex locally Lipschitz loss function. Assume that the estimation procedure $D\mapsto \hatG (\cdot|\cdot) $ is measurable. Then for every $D\in (\Z\times \T\times\{0,1\})^n$ the function
$L^n_D:(\Z\times \T\times\{0,1\})\times\R \mapsto \R$ is measurable.
\end{lem}
\begin{proof}
By Remark~\ref{rem:survival_greater than_zero}, the function $\hatG (u|z)\mapsto 1/\hatG  (u|z)$ is well defined. Since by definition, both $Y$ and $L$ are measurable, we obtain that $(u,z,\delta)\mapsto \delta L(Y(u),z)/\hatG  (u|z)$ is measurable.
\end{proof}

We define the \emph{empirical censored SVM decision
function} to be
\begin{align}\label{eq:svm_censored_decision_function}
 \fc=\argmin_{f\in H}\lambda \|f\|_H^2+\Risk_{L^n_D,D}(f)\equiv \argmin_{f\in H}\lambda \|f\|_H^2+ \frac 1n \sum
L^n_D\big(Z_i,U_i,\delta_i,f(Z_i)\big)\,.
\end{align}
The existence and uniqueness of the empirical censored SVM decision function is ensured by the following lemma:
\begin{lem}\label{lem:unique_cSVM}
Let $L$ be a convex locally Lipschitz loss function. Let $H$ be a separable RKHS of a bounded measurable kernel on $\Z$. Then there exists a unique empirical censored SVM decision function.
\end{lem}
\begin{proof}
Note that given $D$, the loss function
$L^n_D(z,u,\delta,\cdot)$ is convex for every fixed
$z$, $u$, and $\delta$. Hence, the result follows from
Lemma~5.1 together with Theorem~5.2 of \citetalias{SVR}.
\end{proof}

Note that the empirical censored SVM decision function is just the empirical SVM decision function of~\eqref{eq:svm_decision_function}, after replacing  the loss function $L$ with the loss function $L^n_D$. However, there are two important implications to this replacement. Firstly, empirical censored SVM decision functions are obtained by minimizing a different loss function for each given data set. Secondly, the second expression in the minimization problem \eqref{eq:svm_censored_decision_function}, namely, \begin{align*}
\Risk_{L^n_D,D}(f)\equiv \frac 1n \sum_{i=1}^n
L^n_D\big(Z_i,U_i,\delta_i,f(Z_i)\big),
\end{align*}
is no longer constructed from a sum of i.i.d.\ random variables.

We would like to show that the learning method defined by the empirical censored SVM decision functions is indeed a learning method. We first define the term learning method for right censored data or \emph{censored learning method} for short.
\begin{defn}
A censored learning method $\mathfrak{L}^c$ on $\Z\times\T$ maps every data set $D\in(\Z\times\T\times\{0,1\})^n$, $n\geq 1$, to a function $f_D:\Z\mapsto\R$.
\end{defn}
Choose $0<\lambda_n<1$ such that $\lambda_n\rightarrow 0$.
Define the \emph{censored SVM learning method} $\mathfrak{L}^c$,
as $\mathfrak{L}^c(D)=f^c_{D,\lambda_n}$ for all $n\geq 1$.
The measurability of the censored SVM learning method
$\mathfrak{L}^c$ is ensured by the following lemma, which is
an adaptation of Lemma 6.23 of \citetalias{SVR} to the
censored case.
\begin{lem}\label{lem:Lc_is_measurable}
Let $L$ be a convex locally Lipschitz loss function. Let $H$ be a separable RKHS of a bounded measurable kernel on $\Z$. Assume that the estimation procedure $D\mapsto \Ghat_n(\cdot|\cdot) $ is measurable. Then the censored SVM learning method $\mathfrak{L}^c$ is measurable, and the map $D\mapsto f^c_{D,\lambda_n}$ is measurable.
\end{lem}
\begin{proof}
First, by Lemma~2.11 of \citetalias{SVR}, for any $f\in H$,
the map $(z,u,f)\mapsto L(z,Y(u),f(z))$ is measurable. The
survival function $\hatG $ is measurable on
$(\Z\times\R\times\{0,1\})^n\times (\Z\times\R)$ and by
Remark~\ref{rem:survival_greater than_zero}, the function
$D\mapsto\delta_i/\hatG  (u_i|z_i)$ is well defined and
measurable. Hence $D\mapsto  n^{-1}\sum_{i=1}^n\frac{\delta_i
L(z_i,Y(u_i),f(z_i))}{\hatG  (u_i|z_i)}$ is measurable. Note
that the map $f\mapsto \lambda_n\|f\|_H^2$ where $f\in H$
is also measurable. Hence we obtain that the map
$\phi:(\Z\times \T\times\{0,1\})^n\times H \mapsto \R$,
defined by
\begin{align*}
\phi(D,f)=\lambda \|f\|_H^2+\Risk_{L^n_D,D}(f)\,,
\end{align*}
is measurable. By Lemma~\ref{lem:unique_cSVM}, $f^c_{D,\lambda_n}$
is the only element of $H$ satisfying
\begin{align*}
\phi(D,f^c_{D,\lambda_n})=\inf_{f\in H}\phi(D,f)\,.
\end{align*}
By Aumann's measurable selection principle
\citepalias[][Lemma A.3.18]{SVR}, the map $D\mapsto
f^c_{D,\lambda_n}$ is measurable with respect to the
minimal completion of the product $\sigma$-field on
$(\Z\times \T\times\{0,1\})^n$. Since the evaluation map
$(f,z)\mapsto f(z)$ is measurable
\citepalias[][Lemma~2.11]{SVR}, the map $(D,z)\mapsto
f^c_{D,\lambda_n}(z) $ is also measurable.
\end{proof}

\section{Theoretical Results}\label{sec:main}
In the following, we discuss some theoretical results regarding the censored SVM learning method proposed in Section~\ref{sec:censoring}. In Section~\ref{subsec:clipping} we discuss function clipping which will serve as a tool in our analysis. In Section~\ref{subsec:finite_sample} we discuss finite sample bounds. In Section~\ref{subsec:consistency} we discuss consistency. Learning rates are discussed in Section~\ref{subsec:rates}. Finally, censoring model misspecification is discussed in Section~\ref{subsec:misspecification}.

\subsection{Clipped Censored SVM Learning Method}\label{subsec:clipping}
In order to establish the theoretical results of this section we first need to introduce the concept of clipping. We say that a loss
function $L$ can be clipped at $M>0$, if, for all
$(z,y,s)\in\Z\times\Y\times\R$,
\begin{align*}
 L(z,y,\clip{s})\leq L(z,y,s)
\end{align*}
where $\clip{s}$ denotes the clipped value of $s$ at $\pm
M$, that is,
\begin{align*}
   \clip{s}=
   \left\{
     \begin{array}{cl}
       -M & \quad\text{if }s\leq -M \\
       s &  \quad\text{if }-M<s<M \\
       M & \quad\text{if }s\geq M \\
     \end{array}
   \right.
\end{align*}
\citepalias[see][Definition~2.22]{SVR}. The loss functions
$L_{\mathrm{HL}}$, $L_{\mathrm{LS}}$,
$L_{\mathrm{AD}}$, and $L_{\alpha}$ can be clipped at some
$M$ when $\Y=\T$ or $\Y=\{-1,1\}$ \citepalias[Chapter~2]{SVR}.

In our context the response variable $Y$ usually takes it values in a bounded set (see Section~\ref{sec:no_censoring}). When the response space is bounded, we have the following criterion for clipping. Let $L$ be a distance-based loss function, i.e., $L(z,y,s)=\phi(s-y)$ for some function $\phi$. Assume that $\lim_{r\rightarrow\pm\infty}\phi(r)=\infty$. Then $L$ can be clipped at some $M$ \citepalias[Chapter~2]{SVR}.

Moreover, when the sets $\Z$ and $\Y$ are compact, we have the following criterion for clipping which is usually easy to check.
\begin{lem}\label{lem:clipping}
Let $\Z$ and $\Y$ be compact. Let $L:\Z\times\Y\times\R\mapsto [0,\infty)$ be continuous and strictly convex, with a bounded minimizer for every $(z,y)\in\Z\times\Y$. Then $L$ can be clipped at some $M$.
\end{lem}
See proof in Appendix~\ref{sec:additional_proofs}.

%For a loss function $L$ that can be
%clipped at $M$, let
%\begin{align}\label{eq:sup_bound_clipped}
%    B=\max_{z\in\Z,y\in Y, s\in [-M,M]}L(z,y,s)\,.
%\end{align}

For a function $f$, we define $\clip{f}$ to be the clipped version of $f$, i.e., $\clip{f}=\max\{-M,\min\{M, f\}\}$. Finally, we note that the clipped censored SVM learning method, that maps every data set $D\in(\Z\times\T\times\{0,1\})^n$, $n\geq 1$, to the function $\clipfc$ is measurable, where $\clipfc$ is the clipped version of $\fc$ defined in~\eqref{eq:svm_censored_decision_function}. This follows from Lemma~\ref{lem:Lc_is_measurable}, together with the measurability of the clipping operator.

\subsection{Finite Sample Bounds}\label{subsec:finite_sample}
We would like to establish a finite-sample bound for the generalization of clipped censored SVM learning methods. We first need some notation.
Define the censoring estimation error
\begin{align*}
    Err_n(t,z)= \hatG(t|z)-G(t|z)\,,\qquad (t,z)\in\T\times\Z
\end{align*}
to be the difference between the estimated and true
survival functions of the censoring variable.

Let $H$ be an RKHS over the covariates space $\Z\subset\R^d$. Define the $n$-th dyadic entropy number \mbox{$e_n(H,\|\cdot\|_H)$} as the infimum over $\eps$, such that $H$ can be covered with no more than $2^{n-1}$ balls of radius $\eps$ with respect to the metric induced by the norm. For a bounded linear transformation \mbox{$S:H\mapsto F$} where $F$ is a normed space, we define the dyadic entropy number $e_n(S)$ as \mbox{$e_n(SB_H,\|\cdot\|_F)$}. For details, the reader is referred to Appendix~5.6 of \citetalias{SVR}.

Define the Bayes risk $\RLPS=\inf_{f}\RLP(f)$, where the infimum is taken over all measurable functions $f:\Z\mapsto\R$. Note that Bayes risk is defined with respect to both the loss $L$ and the distribution $P$. When a function $\fps$ exists such that $\RLP(\fps)=\RLPS$ we say that $\fps$ is a Bayes decision function.

We need the following assumptions:
\begin{enumerate}
\renewcommand{\labelenumi}{(B\arabic{enumi})}

\item The loss function $L:\Z\times\Y\times\R\mapsto[0,\infty)$ is a locally Lipschitz continuous loss function that can be clipped at $M>0$ such that the supremum bound
    \begin{align}\label{eq:sup_bound}
      L(z,y,s)\leq B
   \end{align}
holds for all $z,y,s\in \Z\times\Y\times[-M,M]$ and for some $B>0$. Moreover, there is a constant $q>0$ such that
\begin{align*}
  |L(z,y,s)-L(z,y,0)|\leq c|s|^q
\end{align*}
for all $z,t,s\in \Z\times\Y\times\R$ and for some $c>0$.
\label{as:LocallyLipchitzClippable}

\item \label{as:var_bound}$H$ is a separable RKHS of a measurable kernel over $\Z$ and $P$ is a distribution over $\Z\times\T$ for which there exist constants $\vartheta\in[0,1]$ and $V>B^{2-\vartheta}$ such that
\begin{align}\label{eq:var_bound}
  P\left(L\circ\clip{f}-L\circ\fps\right)^2\leq VP\left(L\circ\clip{f}-L\circ\fps\right)^{\vartheta}
\end{align}
for all $z,y,s\in \Z\times\Y\times[-M,M]$ and $f\in H$; and where $L\circ f$ is shorthand for the function $(z,y)\mapsto L(z,y,f(z))$.

\item There are constants $a>1$ and $0<p<1$, such that for
    for all $i\geq 1$ the following entropy bound holds:\label{as:entropy_bound1}
\begin{align}\label{eq:entropy_bound}
P[ e_i(\textrm{id}:H\mapsto L_2(\ep_n))]\leq ai^{-\frac{1}{2p}}\,,
\end{align}
where $\textrm{id}:H\mapsto L_2(\ep_n)$ is the embedding of $H$ into the space of square integrable functions with respect to the empirical measure $\ep_n$.

\end{enumerate}

Before we state the main result of this section, we present some
examples for which the assumptions above hold:

\begin{rem}
When $\Y$ is contained in a compact set, Assumption~(B\ref{as:LocallyLipchitzClippable}) holds with $q=1$ for $L_{\mathrm{HL}}$, $L_{\mathrm{AD}}$ and $L_{\alpha}$ and with $q=2$ for $L_{\mathrm{LS}}$ (recall the definitions of the loss functions from Section~\ref{subsec:loss}).
\end{rem}

\begin{rem}
Assumption~(B\ref{as:var_bound}) holds trivially for $\vartheta=0$ with $V=B^2$. It holds for
$L_{\mathrm{LS}}$ with $\vartheta=1$ for compact $\Y$ \citepalias[Example~7.3]{SVR}. Under some conditions on the distribution, it also holds for
$L_{\mathrm{AD}}$ and $L_{\alpha}$ \citepalias[Eq.~9.29]{SVR}.
\end{rem}

\begin{rem}
When $\Z\subset\R^d$ is compact, the entropy bound \eqref{eq:entropy_bound} of
Assumption~(B\ref{as:entropy_bound1}) is satisfied for smooth kernels such as the polynomial and Gaussian kernels for all $p > 0$
\citepalias[see][Section~6.4]{SVR}. The assumption also holds for Gaussian kernels over $\R^d$ for distributions $P_Z$ with positive tail exponent \citepalias[see][Section~7.5]{SVR}.
\end{rem}

We are now ready to establish a finite sample bound for the clipped censored SVM learning methods:
\begin{thm}\label{thm:main}
Let $L$ be a loss function and $H$ be an RKHS such that assumptions (B\ref{as:LocallyLipchitzClippable})--(B\ref{as:entropy_bound1}) hold. Let $f_0\in H$ satisfy $\|L\circ f_0\|_{\infty}\leq B_0$ for some $B_0\geq B$.
Let $\hatG  (t|Z)$ be an
estimator of the survival function of the censoring
variable and assume
(A\ref{as:positiveRisk})--(A\ref{as:T_independent_C}). Then,
for any fixed regularization constant $\lambda>0$, $n\geq 1$,
and $\eta>0$, with probability not less than
$1-3e^{-\eta}$,
\begin{align*}
\lambda \nhs{\fc} +\RLP(\clipfc)-\RLPS\leq&
 3( \lambda\nhs{f_0}+\RLP(f_0)-\RLPS)+3\left(\frac{72\tV\eta}{n}\right)^{1/(2-\vartheta)}
\\
&+\frac{8B_0\eta}{5Kn}+\frac{3B}{K^2}\ep_n Err_n+W\left(\frac{a^{2p}}{\lambda^p n}\right)^{\frac1{2-p-\vartheta+\vartheta p}}
\,,
\end{align*}
where $W$ is a constant that depends only $p$, $M$, $B$, $\vartheta$, $V$ and $K$.
\end{thm}
The proof appears in Appendix~\ref{sec:thm_main}.

For the Kaplan-Meier estimator (see Example~\ref{ex:KM}) bounds of the random error
$\|Err_n\|_{\infty}$ were established \citep{Bitouze99}. In this case we can replace
the bound of Theorem~\ref{thm:main} with a more
explicit one.

Specifically, let $\hatG$ be the Kaplan-Meier estimator. Let $0<K_S=P(T\geq \tau)$ be a lower bound on the survival
function at $\tau$. Then, for every $n\geq 1$ and
$\eps>0$ the following Dvoretzky-Kiefer-Wolfowitz-type inequality holds \citep[][Theorem~2]{Bitouze99}:
\begin{align*}%\label{eq:DKW}
P(
\|\hatG-G\|_{\infty}>\eps)&<\frac{5}{2}\exp\{-2nK_S^2\eps^2+D_o\sqrt{n}K_S\eps\}\,,
\end{align*}
where $D_o$ is some universal constant \citep[see][for a bound on $D_o$]{Wellner07}.
%Write
%\begin{align*}
%  \eps=\frac{\sqrt{\frac{\eta}{2}+\left(\frac{D_o}{4}\right)^2}+\frac{D_o}{4}}{K_S\sn}\,
%\end{align*}
%and note that $\eta=2n K_S^2\eps^2-\sn K_S D_o\eps$
%Using the fact that $\sqrt{x+y}\leq\sqrt{x}+\sqrt{y}$ we obtain that
%\begin{align*}
% 2\frac{\sqrt{\frac{\eta}{2}+\left(\frac{D_o}{4}\right)^2}+\frac{D_o}{4}}{K_S\sn}\leq \frac{\sqrt{2\eta}+D_o}{K_S\sn}\,.
%\end{align*}
%Consequently,
Some algebraic manipulations then yield
\citep{GK_concentration} that for every $\eta>0$ and $n\geq 1$
\begin{align}\label{eq:KM_exp_bound}
  P\left( \|\hatG-G\|_{\infty}>\frac{\sqrt{2\eta}+D_o}{K_S\sn}\right)&<\frac{5}{2}e^{-\eta}\,.
\end{align}

As a result, we obtain the following corollary:
\begin{cor}
Consider the setup of Theorem~\ref{thm:main}. Assume that the censoring variable $C$ is independent of
both $T$ and $Z$. Let $\hatG$ be the Kaplan-Meier estimator of
$G$. Then for any fixed regularization constant $\lambda$, $n\geq 1$, and $\eta>0$, with probability not less than $1-\frac{11}{2}e^{-\eta}$,
\begin{align*}
\lambda \nhs{\fc} +\RLP(\clipfc)-\RLPS \leq& 3( \lambda\nhs{f_0}+\RLP(f_0))+3\left(\frac{72\tV\eta}{n}\right)^{1/(2-\vartheta)}
\\
&+\frac{8B_0\eta}{5Kn}+\frac{\sqrt{18\eta}+3D_o}{K_S K^2\sqrt{n}}+W\left(\frac{a^{2p}}{\lambda^p n}\right)^{\frac1{2-p-\vartheta+\vartheta p}}
\,,
\end{align*}
where $W$ is a constant that depends only on $p$, $M$, $B$, $\vartheta$, $V$ and $K$.
\end{cor}

%Before we continue to the proof, we have the following immediate
%corollary.
%\begin{cor}
%Let $L:\Z\times\Y\times\R\mapsto [0,\infty)$ be a convex, locally
%Lipschitz continuous loss function satisfying $L(Z,Y,0)\leq 1$ for
%all $(z,t)\in \Z\times [0,\tau]$. Let $H$ be a separable RKHS over
%$\Z$ with a continuous kernel $k$ satisfying $\|k\|_{\infty}\leq 1$,
%such that $H$ is dense in $C(\Z)$, the space of continuous function
%from $\Z$ to $\R$. Let $(Z,T,C)$ be a distributed according to $P$.
%Let $\hatG  (t|Z)$ be a consistent estimator of the survival
%function of the censoring variable. Assume (A\ref{as:positiveRisk})
%Then
%\begin{align*}
%    \lim_{n\rightarrow\infty}P(\RLP(\fc)<\RLP^*+\eps)=0\,.
%\end{align*}
%\end{cor}
%The proof follows from the following remark.
%\begin{rem}
%Continuous kernel $k$ for which its RKHS $H$ is dense in $C(\Z)$ is
%called universal. Examples for universal kernels include the
%Gaussian kernels, the exponential kernels, and the binomial kernels.
%For more details, the readers is refereed to
%\citetalias[Chapter~4.6]{SVR}. For universal kernels, we have $\inf_{f\in
%H} \RLP=\RLP^*$ \citepalias[Corollary~5.29]{SVR}.
%\end{rem}

%and $L^n_D(z,u,\delta,0)\leq 1$, we obtain that
%$\|\fp\|_{\infty}\leq \nh{\fp}\leq \lambda^{-1/2}$ and similarly for
%$\fc$ \citepalias[][Lemma 4.23 and Eq.~5.4]{SVR}
%
%Since $\|\fc\|_{H}\leq \lambda^{-1/2}$ and $\|\fp\|_{H}\leq$
%\begin{align*}
% \lambda \nhs{\fc}&+\Risk_{L,P}(\fc)-\inf_{f\in
%H}\Risk_{L,P}(f)-A_2(\lambda)
%\\
%\leq \sup_{\|f\|_H\leq \lambda{-1/2}}|\RLP(f)-\RLNG(f)|
%\end{align*}

\subsection{$\mathcal{P}$-universal Consistency}\label{subsec:consistency}
In this section we discuss consistency of the
clipped version of the censored SVM learning method $\mathfrak{L}^c$ proposed in Section~\ref{sec:censoring}. In general,
$P$-consistency means
that~\eqref{eq:universal_consistency} holds for all $\eps>0$. Universal consistency means that the learning method is $P$-consistent for every probability measure $P$ on $\Z\times\T\times\{0,1\}$. In the following we discuss a more restrictive notion than universal consistency, namely $\mathcal{P}$-universal consistency. Here, $\mathcal{P}$ is
the set of all probability distributions for which there is
a constant $K$ such that conditions~(A\ref{as:positiveRisk})--(A\ref{as:T_independent_C}) hold. We say that a censored learning method is $\mathcal{P}$-universally consistent if~\eqref{eq:universal_consistency} holds for all $P\in\mathcal{P}$. We note that when the first assumption is violated for a set of covariates $\Z_0$ with positive probability, there is no hope of learning the optimal function for all $z\in \Z$, unless some strong assumptions on the model are enforced. The second assumption is required for proving consistency of the learning method $\Lc$ proposed in Section~\ref{sec:censoring}. However, it is possible that other censored learning techniques will be able to achieve consistency for a larger set of probability measures.

%This
%bound depends on four
%different terms: the approximation error $A_2(\lambda)=\lambda \nhs{\fp} +\RLP(\fp)-\RLPS$ wh, the entropy
%of the ball $B_H$, the (locally) continuous Lipschitz
%constant $c_L$, and the error in the estimation of the
%survival function $G$.

In order to show $\mathcal{P}$-universal consistency, we
utilize the bound given in Theorem~\ref{thm:main}. We need the following additional assumptions:
\begin{enumerate}
\renewcommand{\labelenumi}{(B\arabic{enumi})}
\addtocounter{enumi}{3}
\item For all distributions $P$ on $\Z$, $\inf_{f\in H} \RLP(f)=\RLPS$.\label{as:dense}

\item $\hatG $ is consistent for $G$ and there is a finite constant $s>0$ such that
    $P(\|Err_n\|_{\infty}\geq b n^{-1/s})\rightarrow 0$
    for any $b>0$.\label{as:consistency_of_G}

\end{enumerate}

Before we state the main result of this section, we present some
examples for which the assumptions above hold:
\begin{rem}
Assumption~(B\ref{as:dense}) holds when the loss function $L$ is Lipschitz continuous and the RKHS $H$ is dense in $L_1(\mu)$ for all distribution $\mu$ on $\Z$, where $L_1(\mu)$ is the space of equivalence classes of integrable functions. \citepalias[see Theorem~5.31]{SVR}.
\end{rem}
\begin{rem}\label{rem:universal}
Assume that $\Z$ is compact. A continuous kernel $k$ whose corresponding RKHS $H$ is dense in
the class of continuous functions over the compact set $\Z$ is called universal. Examples of universal kernels
include the Gaussian kernels, and other Taylor kernels. For
more details, the reader is referred to
\citetalias[Chapter~4.6]{SVR}. For universal kernels, Assumption~(B\ref{as:dense}) holds for $L_{\mathrm{LS}}$, $L_{\mathrm{HL}}$,
$L_{\mathrm{AD}}$, and $L_{\alpha}$.
\citepalias[Corollary~5.29]{SVR}.
\end{rem}

\begin{rem}
Assume that $\hatG $ is consistent for $G$. When $\hatG $
is the Kaplan-Meier estimator,
Assumption (B\ref{as:consistency_of_G}) holds for all
$s>2$ \citep[Theorem~3]{Bitouze99}. Similarly, when $\hatG
$ is the proportional hazards estimator (see
Example~\ref{ex:ph}), under some conditions, Assumption (B\ref{as:consistency_of_G}) holds
for all $s>2$ \citep[see][Theorem~3.2 and its
conditions]{Cox}. When $\hatG $
is the generalized Kaplan-Meier estimator (see Example~\ref{ex:Generalized_KM}),
under strong conditions on the failure time distribution, \citet{Dabrowska89} showed that Assumption (B\ref{as:consistency_of_G}) holds for all
$s>d/2+2$ where $d$ is the dimension of the covariate space \citep[see][Corollary~2.2 and its
conditions there]{Dabrowska89}. Recently, \citet{GK_concentration} relaxed these assumptions and showed that Assumption (B\ref{as:consistency_of_G}) holds for all $s>2d/\alpha+2$ where $\alpha\in(0,1)$ satisfies \begin{align*}
  \sup_{z_1,z_2\in\Z: \norm{z_1-z_2}\leq h } \left(\sup_{t\in[0,\tau]}|S(t|z_1)-S(t|z_2)|+\sup_{t\in[0,\tau]}|G(t|z_1)-G(t|z_2)|\right)=O(h^{\alpha})\,,
  \end{align*}
where $S(\cdot|z)$ is the survival function of $T$ given $Z=z$ \citep[see][for the conditions]{GK_concentration}.
\end{rem}

Now we are ready for the main result.
\begin{thm}\label{thm:Puniversal_consistency}
Let $L$ be a loss function and $H$ be an RKHS of a bounded kernel over $\Z$. Assume (A\ref{as:positiveRisk})--(A\ref{as:T_independent_C}) and
(B\ref{as:LocallyLipchitzClippable})--(B\ref{as:consistency_of_G}). Let
$\lambda_n\rightarrow 0$, where $0<\lambda_n<1$, and
$\lambda_n^{\max\{q/2,p\}}n\rightarrow \infty$,
where $q$ is defined in
Assumption~(B\ref{as:LocallyLipchitzClippable}). Then
the clipped censored learning method $\mathfrak{L}^c$ is $\mathcal{P}$-universally consistent.
\end{thm}
\begin{proof}
Define the approximation error
\begin{align}\label{eq:A2}
  A_2(\lambda) =\lambda \nhs{\fp} +\RLP(\fp)-\RLPS.
\end{align}
By Theorem~\ref{thm:main}, for $f_0=\fp$ we obtain
\begin{align}\label{eq:bound_witA2}
\begin{split}
\lambda \nhs{\fc} +\RLP(\clipfc)-\RLPS
&\leq 3 A_2(\lambda_n) +3\left(\frac{72\tV\eta}{n}\right)^{1/(2-\vartheta)}
\\
&+\frac{8B_0\eta}{5Kn}+\frac{3B}{K^2}\ep_n Err_n+W\left(\frac{a^{2p}}{\lambda^p n}\right)^{\frac1{2-p-\vartheta+\vartheta p}}
\,,
\end{split}
\end{align}
for any fixed regularization constant $\lambda>0$, $n\geq 1$,
and $\eta>0$, with probability not less than $1-3e^{-\eta}$.

Define $B_0=B+c_o(A_2(\lambda_n)/\lambda_n)^{q/2}$ where $c_o=c(\sup_{z\in\Z}\sqrt{k(z,z)})^{q/2}$ and where $c$ and $q$ are defined in Assumption~(B\ref{as:LocallyLipchitzClippable}). We now show that $\|L\circ\fp\|_{\infty}\leq B_0$. Since the kernel $k$ is bounded, it follows from \citetalias[Lemma~4.23 of][]{SVR} that
$\|\fp\|_{\infty}\leq \sup_{z\in\Z}\sqrt{k(z,z)}\nhs{\fp}$. By the definition of $A_2(\lambda)$, $ \|\fp\|_H \leq (A_2(\lambda)/\lambda)^{1/2}$. Note that for all $(z,y)\in\Z\times\Y$
\begin{align*}
  L(z,y,\fp(z))\leq L(x,y,0)+|L(z,y,\fp(z))-L(x,y,0)|\leq B+c|\fp(z)|^q\,.
\end{align*}
Thus
\begin{align}\label{eq:B0}
\|L\circ\fp\|_{\infty} \leq B+c\|\fp\|_{\infty}^q
\leq B+ c (\sup_{z\in\Z}\sqrt{k(z,z)}\|\fp\|_H)^q
\leq B+c_o \left(\frac{A_2(\lambda)}{\lambda}\right)^{\frac{q}{2}}=B_0\,.
\end{align}

Assumption~(B\ref{as:dense}), together with Lemma~5.15 of \citetalias{SVR}, shows that $A_2(\lambda_n)$ converges to zero as $n$ converges to infinity. Clearly $3\left(\frac{72\tV\eta}{n}\right)^{1/(2-\vartheta)}$ converges to zero. $8\eta(B+c_o(A_2(\lambda_n)/\lambda_n)^{q/2})/(5Kn)$ converges to zero since $\lambda_n^{q/2}n\rightarrow \infty$. By Assumption~(B\ref{as:consistency_of_G}), $\ep_n Err_n$ converges to zero. Finally, $W\left(\frac{a^{2p}}{\lambda^p n}\right)^{\frac1{2-p-\vartheta+\vartheta p}}$ converges to zero since $\lambda^p n\rightarrow\infty$. Hence, for every fixed $\eta$, the right hand side of~\eqref{eq:bound_witA2} converges to zero, which implies~\eqref{eq:universal_consistency}. Since~\eqref{eq:universal_consistency} holds for every $P\in\mathcal{P}$, we obtain $\mathcal{P}$-universal consistency.
\end{proof}

\subsection{Learning Rates}\label{subsec:rates}
In the previous section we discussed $\mathcal{P}$-universal consistency which ensures that for every probability $P\in\mathcal{P}$, the clipped learning method $\Lc$ asymptotically learns the optimal function. In this section we would like to study learning rates.

We define learning rates for censored learning methods similarly to the definition for regular learning methods \citepalias[see][Definition~6.5]{SVR}:
\begin{defn}
Let $L:\Z\times\Y\times\R\mapsto[0,\infty)$ be a loss function. Let $P\in\mathcal{P}$ be a distribution. We say that a censored learning method $\Lc$ learns with a rate $\{\eps_n\}_n$, where $\{\eps_n\}\subset (0,1]$ is a sequence decreasing to $0$, if for some constant $c_P>0$, all $n\geq 1$, and all $\eta\in [0,\infty)$, there exists a constant $c_\eta\in[1,\infty)$ that depends on $\eta$ and $\{\eps_n\}$ but not on $P$, such that
\begin{align*}
    P(D\in (\Z\times\T\times\{0,1\})^n \,:\, \RLP(\fc)\leq \RLP^*+ c_P c_{\eta}\eps_n)\geq 1-e^{-\eta}\,.
\end{align*}
\end{defn}

In order to study the learning rates, we need an additional assumption:
\begin{enumerate}
\renewcommand{\labelenumi}{(B\arabic{enumi})}
\addtocounter{enumi}{5}
\item There exist constants $c_1$ and $\beta\in(0,1]$ such that
  $ A_2(\lambda)\leq c_1\lambda^{\beta}$ for all $\lambda\geq 0$, where $A_2$ is the approximation error function defined in~\eqref{eq:A2}.\label{as:approx_error}
\end{enumerate}

\begin{lem}\label{lem:learning_rates}
Let $L$ be a loss function and $H$ be an RKHS of a bounded kernel over $\Z$. Assume (A\ref{as:positiveRisk})--(A\ref{as:T_independent_C}) and
(B\ref{as:LocallyLipchitzClippable})--(B\ref{as:approx_error}). Then
the learning rate of the clipped $\mathfrak{L}^c$ is given by
\begin{align*}
  n^{-\min\left\{\frac{2\beta}{q+(2-q)\beta},\frac{\beta}{(2-p-\vartheta+\vartheta p)\beta+p},\frac1s\right\}}
\end{align*}
where $q$, $\vartheta$, $p$, $s$, and~$\beta$, are as defined in Assumptions~(B\ref{as:LocallyLipchitzClippable}),~(B\ref{as:var_bound}),~(B\ref{as:entropy_bound1}),~(B\ref{as:consistency_of_G}), and~(B\ref{as:approx_error}), respectively.
\end{lem}
Before we provide the proof, we derive learning rates for two specific examples.
\begin{example}
\textbf{Fast Rate:} Assume that the censoring mechanism is known, the loss function is the square loss, the kernel is Gaussian, $\Z$ is compact, $\Y$ is bounded, and let $\beta<1$. It follows that (B\ref{as:LocallyLipchitzClippable})
~holds for $q=2$, (B\ref{as:var_bound})~holds for $\vartheta=1$ \citepalias[Example 7.3]{SVR}, (B\ref{as:entropy_bound1})~holds for all $0<p<1$ \citepalias[Theorem 6.27]{SVR}, and (B\ref{as:consistency_of_G})~holds for all $s>0$. Thus the obtained rate is $n^{-\beta+\eps}$, where $\eps>0$ is an arbitrarily small number.
\end{example}

\begin{example}
\textbf{Standard Rate:} Assume that the censoring mechanism follows the proportional hazards assumption, the loss function is either $L_{\mathrm{HL}}$, $L_{\mathrm{AD}}$ or $L_{\alpha}$, the kernel is Gaussian, $\Z$ is compact, and let $\beta\geq 1/2$. It follows that (B\ref{as:LocallyLipchitzClippable}) holds for $q=1$, (B\ref{as:var_bound})~holds trivially for $\vartheta=0$, (B\ref{as:entropy_bound1})~holds for all $0<p<1$, and (B\ref{as:consistency_of_G})~holds for all $s>2$. Thus the obtained rate is $n^{-1/2+\eps}$, where $\eps>0$ is an arbitrarily small number.
\end{example}

\begin{proof}[Proof of Lemma~\ref{lem:learning_rates}]
Using Assumption~(B\ref{as:approx_error}) and substituting~\eqref{eq:B0} in~\eqref{eq:bound_witA2} we obtain
\begin{align*}
\Risk_{L,P}(\fc)-\RLPS\leq&
 \quad c_2\max\{\eta,1\}\left(\lambda^{\beta}+n^{-\frac1{2-p-\vartheta+\vartheta p}}\lambda^{-\frac {p}{2-p-\vartheta+\vartheta p}}
 + n^{-1}\lambda^{q(\beta-1)/2}\right)
 \\&
 + 3\left(\frac{72\tV\eta}{n}\right)^{1/(2-\vartheta)}
+\frac{8B\eta}{5Kn}+\frac{3B}{K^2}\ep_n Err_n\,,
\end{align*}
with probability not less than $1-3e^{-\eta}$, for some constant $c_2$ that depends on $p$, $M$, $\vartheta$, $c_1$, $V$, and $K$ but not on $P$. Denote
\begin{align*}
  \rho=\min\left\{\frac{2\beta}{q+(2-q)\beta},\frac{\beta}{(2-p-\vartheta+\vartheta p)\beta+p}\right\}.
\end{align*}
It can be shown that for $\lambda=n^{-\rho/\beta}$, we obtain
\begin{align}\label{eq:find__minimizer_lambda}
  \left(\lambda^{\beta}+\lambda^{-\frac {p}{2-p-\vartheta+\vartheta p}} n^{-\frac1{2-p-\vartheta+\vartheta p}}+ n^{-1}\lambda^{q(\beta-1)/2}\right)\leq 3n^{-\rho}
\end{align}
To see this, denote $\alpha=p$, $\gamma=(2-p-\vartheta+\vartheta p)^{-1}$, $r=2/q$, $s=\lambda$ and $t=n^{-1}$ and note that $\alpha,\beta,\gamma,s,t\in(0,1]$ and that $r>0$. Then apply Lemma A.1.7 of \citetalias{SVR} to bound the LHS of~\eqref{eq:find__minimizer_lambda}, while noting that the proof of this lemma holds for all $r>0$.

By Assumption~(B\ref{as:consistency_of_G}) and the fact that $\|Err_n\|_{\infty}<1$, there exists a constant $c_3=c(\eta)$ that depends only on $\eta$, such that for all $n\geq 1$,
\begin{align*}
    P(\|Err_n\|_{\infty}> c_3 n^{-1/s})<e^{-\eta}\,.
\end{align*}

It then follows that
\begin{align*}
 P\left(\Risk_{L,P}(\fc)-\inf_{f\in H}\RLP(f)\leq c_P c_\eta n^{-\min\{\rho,1/s\}}\right)\geq 1-4e^{-\eta}\,,
\end{align*}
for some constants $c_P$ that depends on $p$, $M$, $\vartheta$, $c$, $B$, $V$, and $K$ but is independent of $\eta$, and $c_\eta$ that depends only on $\eta$.
\end{proof}

\subsection{Misspecified Censoring Model}\label{subsec:misspecification}
In Section~\ref{subsec:consistency} we showed that under conditions
(B\ref{as:LocallyLipchitzClippable})--(B\ref{as:consistency_of_G}) the
clipped censored SVM learning method $\mathfrak{L}^c$ is
$\mathcal{P}$-universally consistent. While one can choose
the Hilbert space $H$ and the loss function $L$ in advance
such that conditions
(B\ref{as:LocallyLipchitzClippable})--(B\ref{as:dense}) hold,
condition~(B\ref{as:consistency_of_G}) need not hold when
the censoring mechanism is misspecified. In
the following, we consider this case.

Let $\hatG(t|z)$ be the estimator of the survival function
for the censoring variable. The deviation of $\hatG(t|z)$
from the true survival function $G(t|z)$ can be divided
into two terms. The first term is the deviation of the
estimator $\hatG(t|z)$ from its limit, while the second
term is the difference between the estimator limit and the
true survival function. More formally, let $G_P(t|z)$ be the limit of the
estimator under the probability measure $P$, and assume it
exists. Define the errors
\begin{align*}
    Err_n(t,z)=Err_{n1}(t,z)+Err_2(t,z)\equiv \Bigl(\hatG(t|z)-G_P(t|z)\Bigr)+\Bigl(G_P(t|z)-G(t|z)\Bigr)\,.
\end{align*}
Note that $Err_{n1}$ is a random function that depends on
the data, the estimation procedure, and the probability
measure $P$, while $Err_2$ is a fixed function that depends
only on the estimation procedure and the probability
measure $P$. When the model is correctly
specified, and the estimator is consistent, the second term
vanishes.

%The proof appears in Appendix~\ref{sec:proof_missepcification}.
\begin{thm}\label{thm:misspecification}
Let $L$ be a loss function and $H$ be an RKHS of a bounded kernel over $\Z$. Assume (A\ref{as:positiveRisk})--(A\ref{as:T_independent_C}) and
(B\ref{as:LocallyLipchitzClippable})--(B\ref{as:dense}). Let
$\lambda_n\rightarrow 0$, where $0<\lambda_n<1$ and
$\lambda_n^{\max\{q/2,p\}}n\rightarrow \infty$. Then, for every fixed $\eps>0$,
\begin{align*}
    \lim_{n\rightarrow\infty} P\left(D\in (\Z\times\T\times\{0,1\})^n \,:\,\RLP(\clipfc)\leq
    \RLPS+\frac{3B}{K^2}\left|P(G_P-G)\right|+\eps\right)=1\,.
\end{align*}
\end{thm}
\begin{proof}
By~\eqref{eq:bound_witA2}, for every fixed $\eta>0$ and $n\geq 1$,
\begin{align}\label{eq:bound_witA2_missp}
\begin{split}
&\lambda \nhs{\fc} +\RLP(\clipfc)-\RLPS\\
&\leq \left(3 A_2(\lambda_n) +3\left(\frac{72\tV\eta}{n}\right)^{1/(2-\vartheta)}
+\frac{8B_0\eta}{5Kn}+\frac{3B}{K^2}\|\hatG-G_P\|_{\infty} +W\left(\frac{a^{2p}}{\lambda^p n}\right)^{\frac1{2-p-\vartheta+\vartheta p}} \right)\\
&+\frac{3}{K^2}\ep_n Err_2\,,
\end{split}
\end{align}
for any fixed regularization constant $\lambda>0$, $n\geq 1$,
and $\eta>0$, with probability not less than $1-3e^{-\eta}$. Since $P(\|\hatG-G_P\|_{\infty}\geq b n^{-1/s})\rightarrow 0
$, it follows from the same arguments as in the proof of Theorem~\ref{thm:Puniversal_consistency}, that the first expression on the RHS of~\eqref{eq:bound_witA2_missp} converges in probability to zero. By the law of large numbers, $\ep_n Err_2\asarrow P(G_P-G)$, and the result follows.
\end{proof}

Theorem~\ref{thm:misspecification} proves that even under misspecification of the censored data model, the clipped censored learning method $\mathfrak{L}^c$ achieves the optimal risk up to a constant that depends on $P(G_P-G)$, which is the expected distance of the limit
of the estimator from the true distribution. If the
estimator estimates reasonably well, one can hope that
this term is small, even under misspecification.

We now show that the additional condition
\begin{align}\label{eq:additional_condition_misspecification}
P(\|\hatG-G_P\|_{\infty}\geq b n^{-1/s})\rightarrow 0
\end{align}
of Theorem~\ref{thm:misspecification} holds for both the
Kaplan-Meier estimator and the Cox model estimator.

\begin{example}\textbf{Kaplan-Meier estimator:} Let $\hatG$ be the Kaplan-Meier estimator of $G$. Let $G_P$
be the limit of $\hatG$. Note that $G_P$ is the
marginal distribution of the censoring variable. It follows
from~\eqref{eq:KM_exp_bound} that
condition~\eqref{eq:additional_condition_misspecification}
holds for all $s>2$.
\end{example}

\begin{example}\textbf{Cox model estimator:}\label{ex:Cox_Convergence}
Let $\hatG$ be the estimator of $G$ when the Cox model is
assumed (see Example~\ref{ex:ph}). Let $G_P$ be the limit
of $\hatG$. It has been shown that the limit $G_P$ exists,
regardless of the correctness of the proportional hazards model \citep{Cox}.
Moreover, for all $\eps>0$, and all $n$ large enough,
\begin{align*}
   P( \|\hatG-G_P\|>\eps)\leq  \exp\{-W_1 n \eps^2+W_2\sn\eps\}\,,
\end{align*}
where $W_1$, $W_2$ are universal constants that depend on
the set $\Z$, the variance of $Z$, the constants $K$ and
$K_S$, but otherwise do not depend on the distribution
$P$ \citep[see][Theorem 3.2, and conditions therein]{Cox}.
Fix $\eta>0$ and write
\begin{align*}
\eps =\frac{\sqrt{\eta+W_2^2}+W_2}{2W_1\sn} \,.
\end{align*}
Some algebraic manipulations then yield
\begin{align*}
  \limsup_{n\rightarrow\infty}  P\left( \|\hatG-G_P\|_{\infty}>\frac{\sqrt{W_1\eta}+W_2}{W_1\sn}\right)&<e^{-\eta}\,.
\end{align*}
Hence,
condition~\eqref{eq:additional_condition_misspecification}
holds for all $s>2$.
\end{example}

\section{Simulation Study}\label{sec:simulation}
\begin{figure}[t!]
%\vskip 0.2in
\begin{center}
\includegraphics[width=0.8\textwidth]{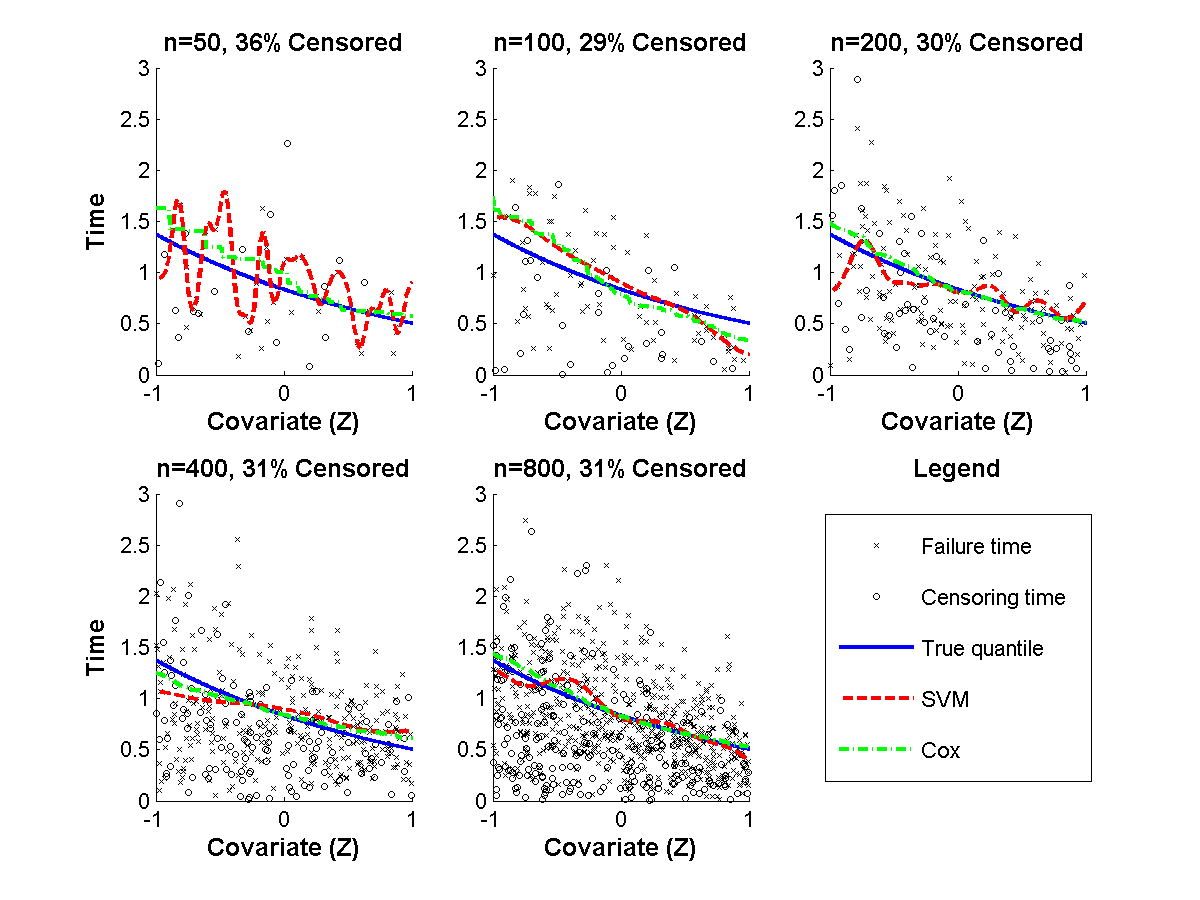}
\caption{Weibull failure time, proportional hazards
(Setting~1): The true conditional median (solid blue), the
SVM decision function (dashed red), and the Cox regression
median (dot-dashed green) are plotted for samples of size
$n=50,100,200,400$ and $800$. The censoring percentage is
given for each sample size. An observed failure times is
represented by an $\times$, and an observed censoring time is
represented by an $\circ$.} \label{fig:Weibull_fig}
\end{center}
%\vskip 0.2in
\end{figure}

\begin{figure}[t!]
%\vskip 0.2in
\begin{center}
\includegraphics[width=0.8\textwidth]{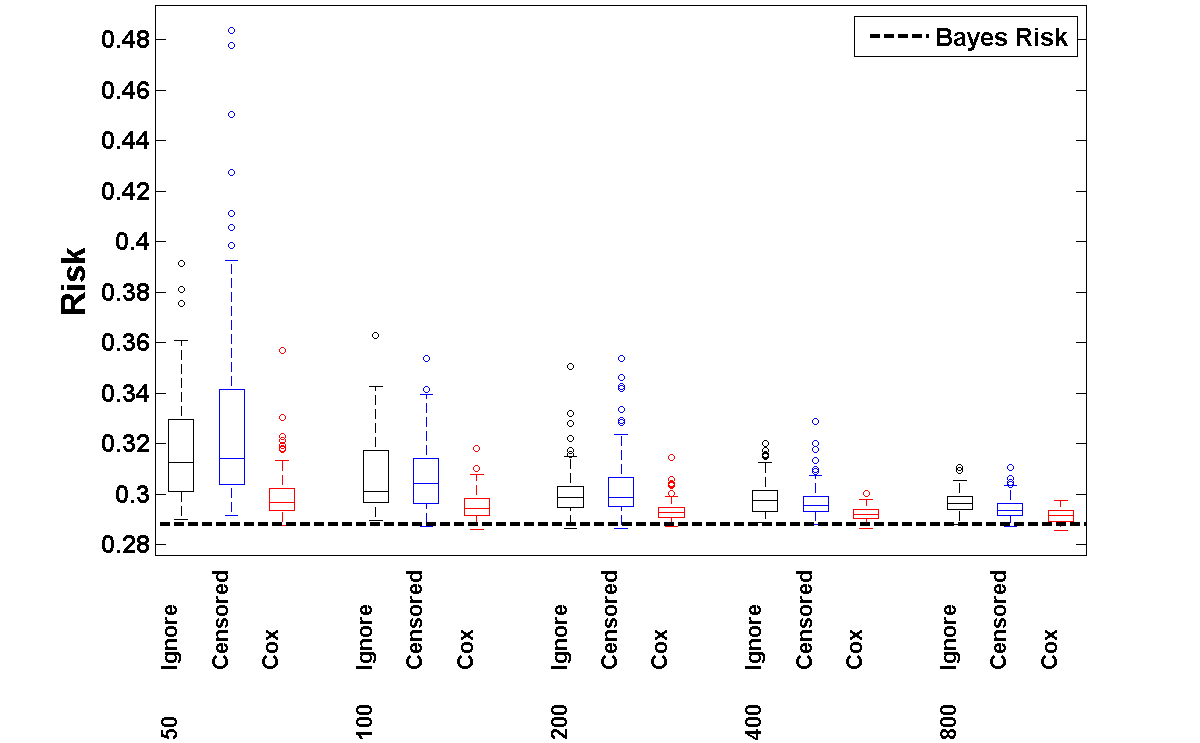}
\caption{Weibull failure time, proportional hazards
(Setting~1): Distribution of the risk for different sizes of data set, for standard SVM that ignores the censored observations
(Ignore), for censored SVM (Censored), and for the Cox
regression median (Cox). Bayes risk is denoted by a black dashed line. Each box plot is based on 100
repetitions of the simulation for each size of data set.} \label{fig:Weibull_boxplot}
\end{center}
%\vskip 0.2in
\end{figure}

\begin{figure}[t!]
%\vskip 0.2in
\begin{center}
\includegraphics[width=0.8\textwidth]{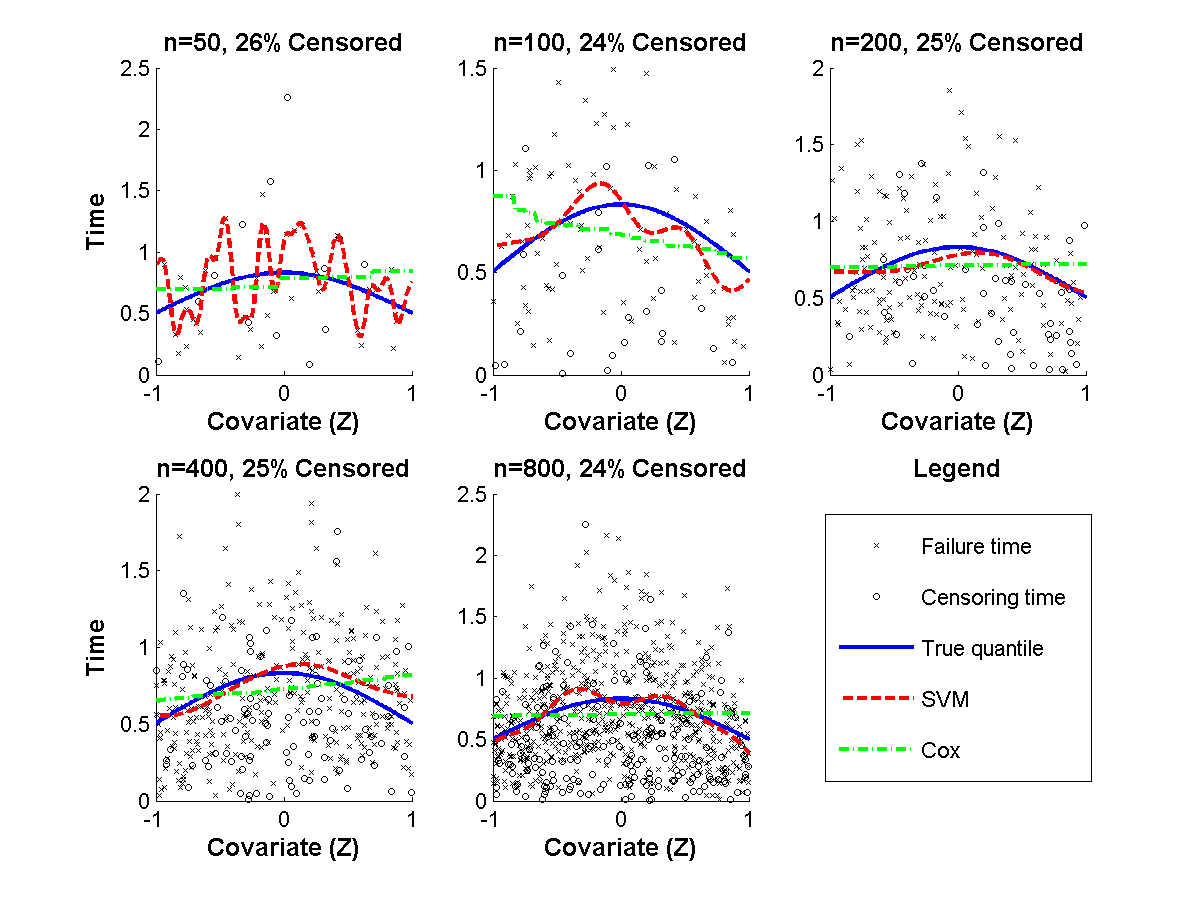}
\caption{Weibull failure time, non-linear proportional hazards (Setting~2): The true
conditional median (solid blue), the SVM decision function
(dashed red), and the Cox regression median (dot-dashed green)
are plotted for samples of size $n=50,100,200,400$ and $800$.
The censoring percentage is given for each sample size. An
observed failure times is represented by an $\times$, and
an observed censoring time is represented by an $\circ$.}
\label{fig:Weibull2_fig}
\end{center}
%\vskip 0.2in
\end{figure}

\begin{figure}[t!]
%\vskip 0.2in
\begin{center}
\includegraphics[width=0.8\textwidth]{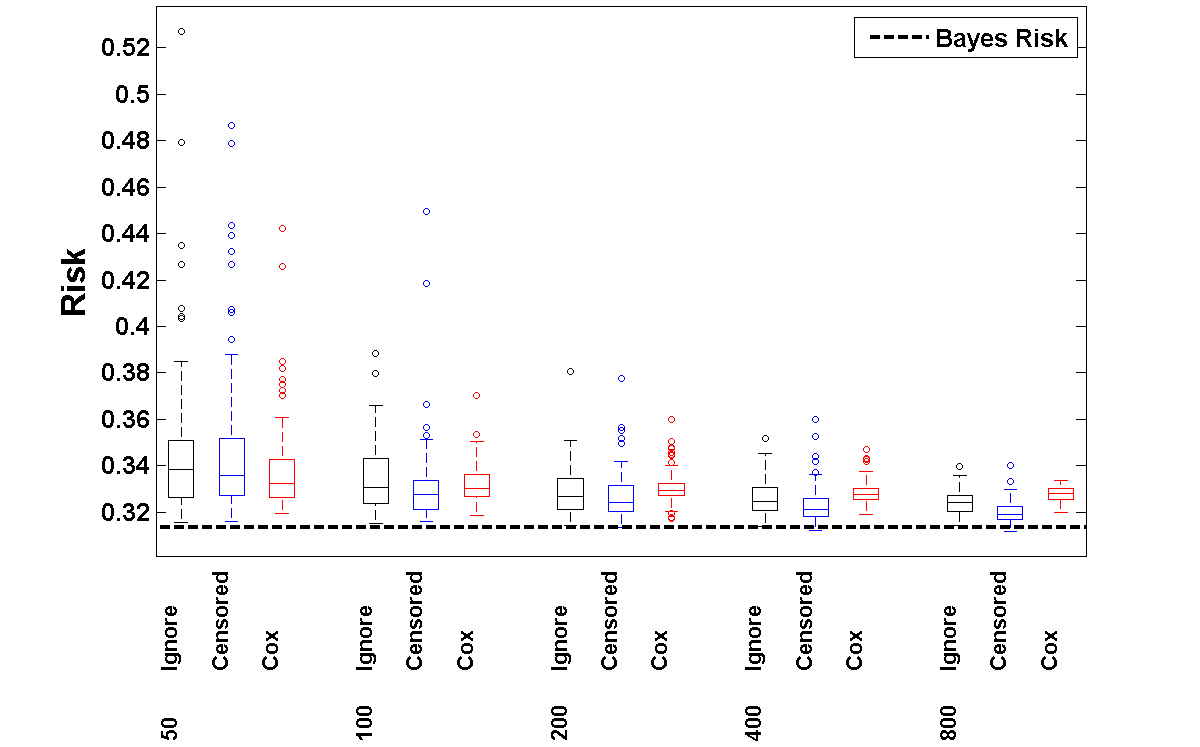}
\caption{Weibull failure time, non-linear proportional hazards (Setting~2): Distribution of the risk for different sizes of data set, for standard SVM that ignores the censored observations
(Ignore), for censored SVM (Censored), and for the Cox
regression median (Cox). Bayes risk is denoted by a black dashed line. Each box plot is based on 100
repetitions of the simulation for each size of data set.} \label{fig:Weibull2_boxplot}
\end{center}
%\vskip 0.2in
\end{figure}

\begin{figure}[t!]
%\vskip 0.2in
\begin{center}
\includegraphics[width=0.8\textwidth]{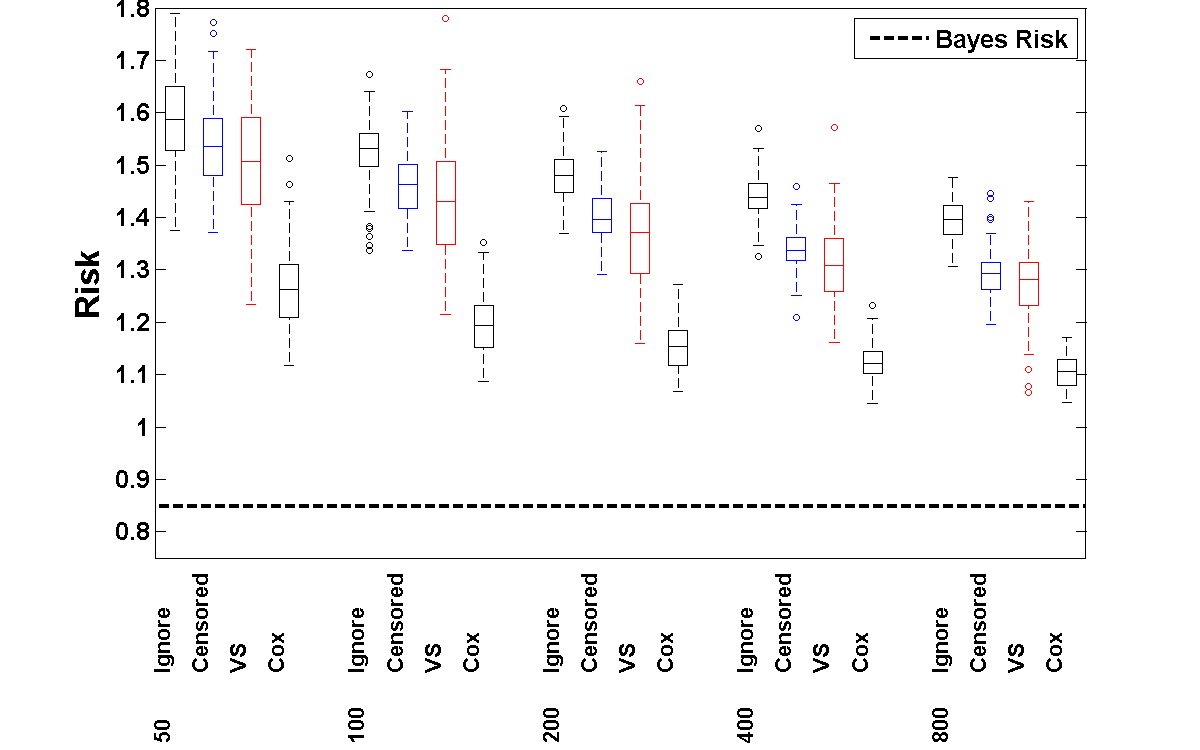}
\caption{Multidimensional Weibull failure time (Setting~3): Distribution of the risk for different data set
sizes, for standard SVM that ignores the censored observations
(Ignore), for censored SVM (Censored), for censored SVM with variable selection (VS), and for the Cox
regression median (Cox). Bayes risk is denoted by a black dashed line. Each box plot is based on 100
repetitions of the simulation for each size of data set.} \label{fig:Weibull_10D_boxplot}
\end{center}
%\vskip 0.2in
\end{figure}
\begin{figure}[t!]
%\vskip 0.2in
\begin{center}
\includegraphics[width=0.8\textwidth]{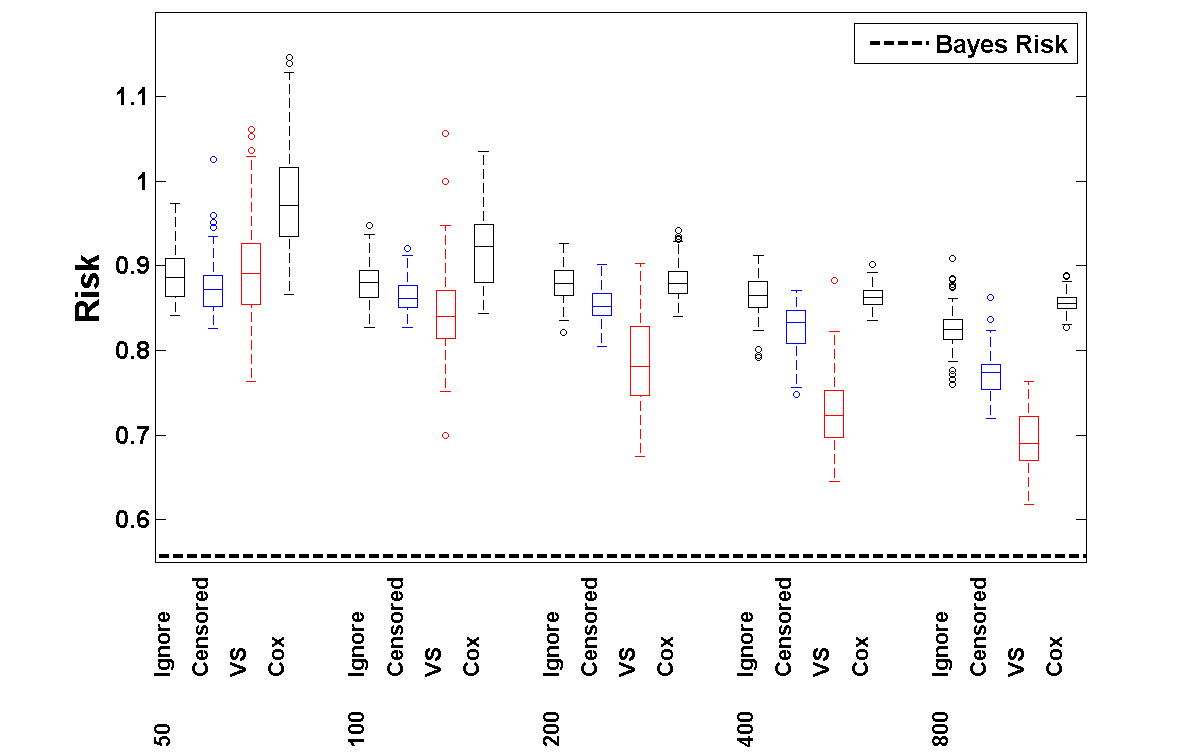}
\caption{Multidimensional Weibull failure time, non-linear proportional hazards (Setting~4): Distribution of the risk for different data set
sizes, for standard SVM that ignores the censored observations
(Ignore), for censored SVM (Censored), for censored SVM with variable selection (VS), and for the Cox
regression median (Cox). Bayes risk is denoted by a black dashed line. Each box plot is based on 100
repetitions of the simulation for each given data set
size.} \label{fig:Weibull2_10D_boxplot}
\end{center}
%\vskip 0.2in
\end{figure}

\begin{figure}[t!]
%\vskip 0.2in
\begin{center}
\includegraphics[width=0.8\textwidth]{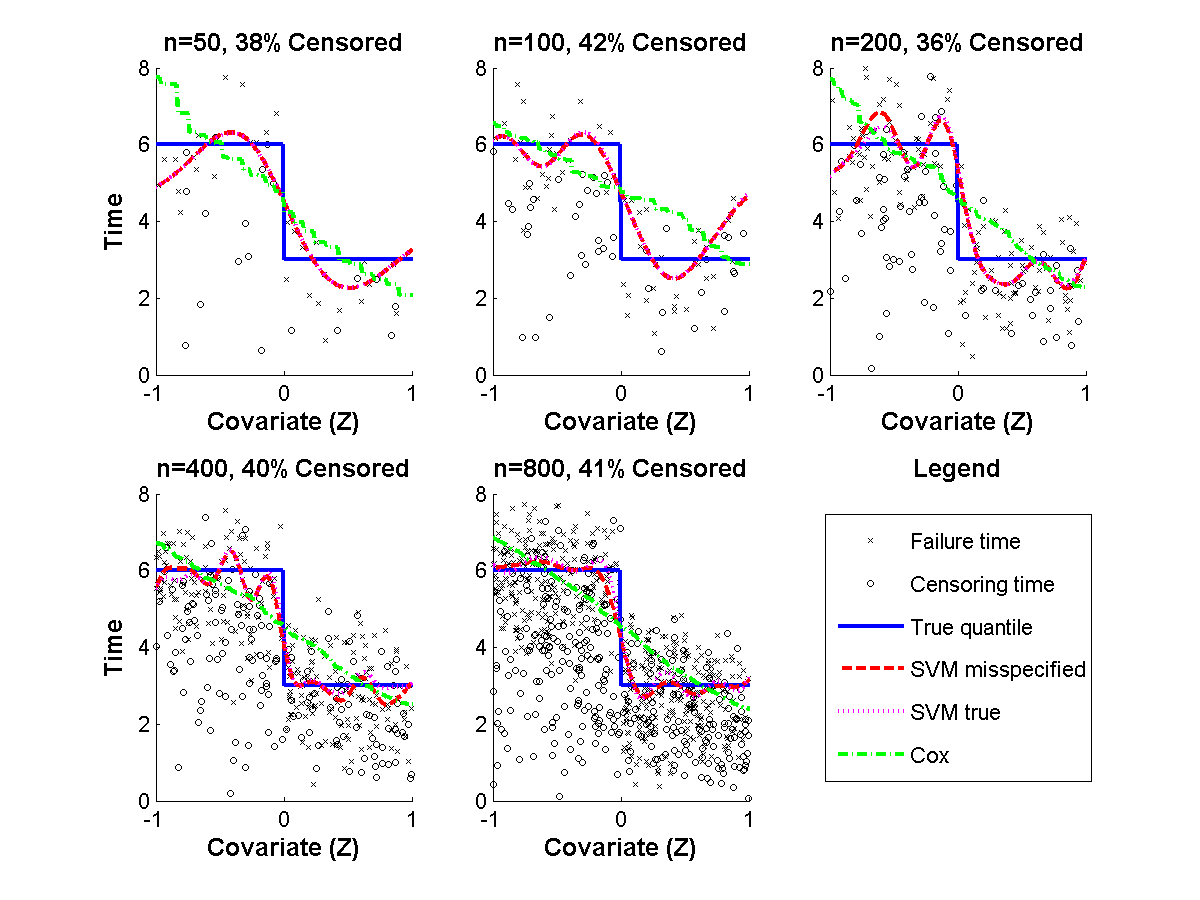}
\caption{Step function median, Weibull censoring time (Setting~5): The true
conditional median (solid blue), the SVM decision function
using the Kaplan-Meier estimator for the censoring (dashed red), the SVM decision function
using the Cox estimator for censoring (doted magenta), and the Cox regression median (dot-dashed green)
are plotted for samples of size $n=50,100,200,400$ and $800$.
The censoring percentage is given for each sample size. An
observed failure times is represented by an $\times$, and
an observed censoring time is represented by an $\circ$.}
\label{fig:Jump_fig}
\end{center}
%\vskip 0.2in
\end{figure}

\begin{figure}[t!]
%\vskip 0.2in
\begin{center}
\includegraphics[width=\textwidth]{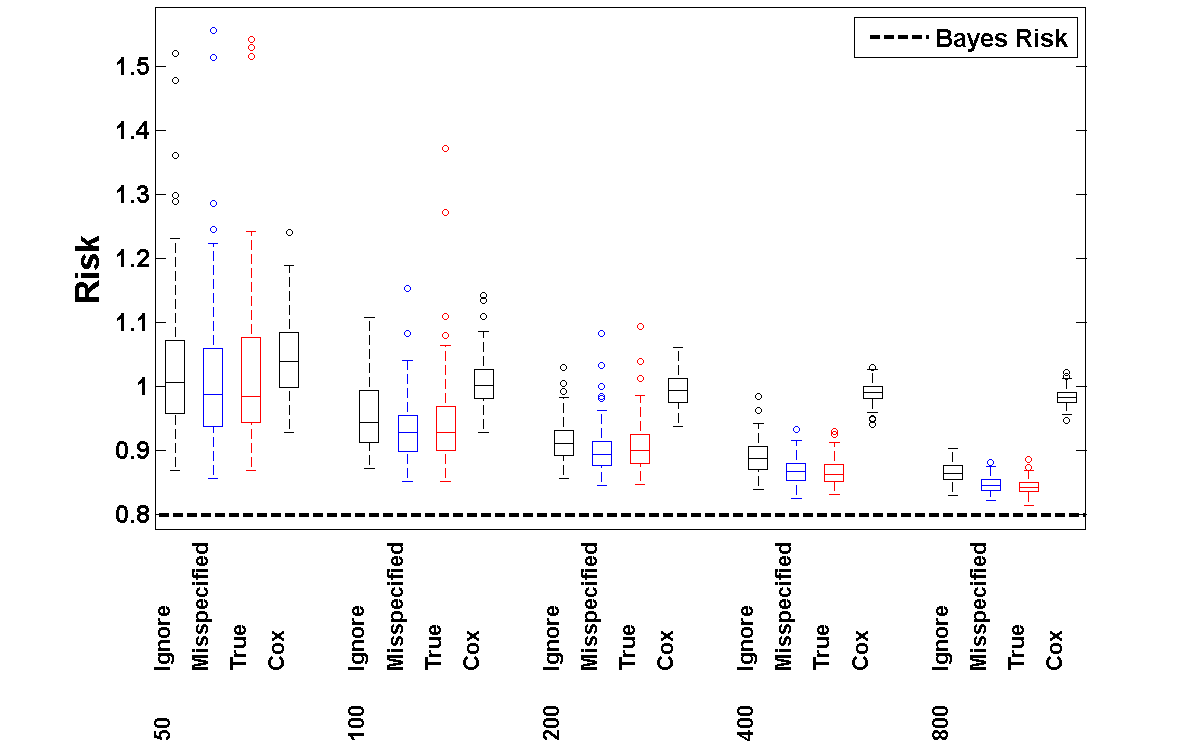}
\caption{Step function median, Weibull censoring time
(Setting~5): Distribution of the risk for different sizes of data
set, for standard SVM that ignores the censored
observations (Ignore), for censored SVM with the Kaplan-Meier
estimator for censoring (Misspecified), for censored
SVM with the Cox estimator for censoring (True), and for the
Cox regression median (Cox). The Bayes risk is denoted by a
black dashed line. Each box plot is based on 100
repetitions of the simulation for each size of data set.} \label{fig:Jump_boxplot}
\end{center}
%\vskip 0.2in
\end{figure}

In this section we illustrate the use of the censored SVM
learning method proposed in Section~\ref{sec:censoring} via
a simulation study. We consider five different data-generating mechanisms, including one-dimensional and
multidimensional settings, and different types of censoring
mechanisms. We compute the censored SVM decision function
with respect to the absolute deviation loss function
$L_{\mathrm{AD}}$. For this loss function, the Bayes risk is given by
the conditional median (see Example~\ref{ex:median}). We
choose to compute the conditional median and not the
conditional mean, since censoring prevents reliable
estimation of the unrestricted mean survival time when no
further assumptions on the tail of the distribution are
made (see discussions in
\citealp{Karrison97,Zucker98,Chen01}). We compare the
results of the SVM approach to the results obtained by the
Cox model and to the Bayes risk. We test the effects of ignoring
the censored observations. Finally, for multidimensional
examples, we also check the benefit of variable selection.

The algorithm presented in Section~\ref{sec:censoring} was
implemented in the Matlab environment. For the
implementation we used the Spider library for
Matlab\footnote{The Spider library for Matlab can be
downloaded form
\url{http://www.kyb.tuebingen.mpg.de/bs/people/spider/}}.
The Matlab code for both the algorithm and the simulations
can be found in \ref{sec:suppA}. The distribution of the
censoring variable was estimated using the Kaplan-Meier
estimator (see Example~\ref{ex:KM}). We used the Gaussian
RBF kernel
$k_{\sigma}(x_1,x_2)=\exp(\sigma^{-2}\|x_1-x_2\|_2^2)$,
where the width of the kernel $\sigma$ was chosen using
cross-validation. Instead of minimizing the regularized
problem \eqref{eq:svm_censored_decision_function}, we solve
the equivalent problem \citepalias[see][Chapter~5]{SVR}:
\begin{align*}
    \text{Minimize }\Risk_{L^n_D,D}(f)  \text{ under the constraint } \|f\|_H^2<\lambda^{-1}\,,
\end{align*}
where $H$ is the RKHS with respect to the kernel $k_{\sigma}$, and $\lambda$ is some constant chosen using
cross-validation. Note that there is no need to compute the norm of the function $f$ in the RKHS space
$H$ explicitly. The norm can be obtained using the kernel matrix $K$
with coefficients $ k_{ij}= k(Z_i, Z_j)$
\citepalias[see][Chapter~11]{SVR}. The risk of the
estimated functions was computed numerically, using a
randomly generated data set of size $10000$.

In some simulations the failure time is distributed
according to the Weibull distribution \citep{Lawless03}.
The density of the Weibull distribution is given by
\begin{align*}
f(t)=
\frac{\kappa}{\rho}\left(\frac{t}{\rho}\right)^{\kappa-1}e^{-(t/\rho)^{\kappa}} \indi{t\geq 0}\,,
\end{align*}
where $\kappa>0$ is the shape parameter and $\rho>0$ is the
scale parameter. Assume that $\kappa$ is fixed and that
$\rho=\exp(\beta_0+\beta'Z)$, where $\beta_0$ is a constant, $\beta$ is the coefficient vector, and $Z$ is the covariate vector.
In this case, the
failure time distribution follows the proportional hazards
assumption, i.e., the hazard rate is given by $h(t|Z)=
\exp(\beta_0+\beta'Z)d\Lambda(t)$, where $\Lambda(t)=t^\kappa$. When the
proportional hazards assumption holds, estimation based on
Cox regression is consistent and efficient (see
Example~\ref{ex:ph}; note that the distribution discussed there is
of the censoring variable and not of the failure time,
nevertheless, the estimation procedure is similar). Thus, when the
failure time distribution follows the proportional hazards
assumption, we use the Cox regression as a benchmark.

In the first setting, the covariates $Z$ are generated
uniformly on the segment $[-1,1]$. The failure time follows
the Weibull distribution with shape parameter $2$ and
scale parameter $-0.5 Z$. Note that the proportional hazards
assumption holds. The censoring variable $C$ is distributed
uniformly on the segment $[0,c_0]$ where the constant $c_0$
is chosen such that the mean censoring percentage is
$30\%$. We used $5$-fold-cross-validation to choose the
kernel width and the regularization constant among the set
of pairs
\begin{align*}
 (\lambda^{-1},\sigma)=(0.1\cdot 10^{i},0.05\cdot 2^{j})\,,\qquad i,j\in\{0,1,2,3\}\,.
\end{align*}
We repeated the simulation $100$ times for each of the sample sizes $50,100,200,400$, and $800$.

In Figure~\ref{fig:Weibull_fig}, the conditional median
obtained by the censored SVM learning method and by Cox
regression are plotted. The true median is plotted as a
reference. In Figure~\ref{fig:Weibull_boxplot}, we compare
the risk of the SVM method to the median of the
survival function obtained by Cox regression (to which we
refer as the Cox regression median). We also examined the effect
of ignoring the censored observations by computing the
standard SVM decision function for the data set in which
all the censored observations were deleted. Both figures
show that even though the SVM does not use the proportional
hazards assumption for estimation, the results are
comparable to those of Cox regression, especially for
larger sample sizes. Figure~\ref{fig:Weibull_boxplot} also
shows that there is a non-negligible price for ignoring the
censored observations.

The second setting differs from the first setting only
in the failure time distribution. In the second setting the
failure time distribution follows the Weibull distribution
with scale parameter $-0.5 Z^2$. Note that the proportional
hazards assumption holds for $Z^2$, but not for the original
covariate $Z$. In
Figure~\ref{fig:Weibull2_fig}, the true, the SVM median, and the
Cox regression median are plotted. In
Figure~\ref{fig:Weibull2_boxplot}, we compare the risk of
SVM to that of Cox regression. Both figures show that
in this case SVM does better than Cox regression.
Figure~\ref{fig:Weibull2_boxplot} also shows the price of
ignoring censored observations.

The third and forth settings are generalizations of the
first two, respectively, to 10-dimensional covariates. The covariates $Z$
are generated uniformly on $[-1,1]^{10}$. The failure time
follows the Weibull distribution with shape parameter $2$.
The scale parameter of the third and forth settings are
$-0.5 Z_1+2Z_2-Z_3$ and $-0.5 (Z_1)^2+2(Z_2)^2-(Z_3)^2$,
respectively. Note that these models are sparse, namely, they
depend only on the first three variables. The censoring
variable $C$ is distributed uniformly on the segment
$[0,c_0]$, where the constant $c_0$ is chosen such that the
mean censoring percentage is $40\%$. We used
$5$-fold-cross-validation to choose the kernel width and
the regularization constant among the set of pairs
\begin{align*}
 (\lambda^{-1},\sigma)=(0.1\cdot 10^{i},0.2\cdot 2^{j})\,,\qquad i,j\in\{0,1,2,3\}\,.   \end{align*}

The results for the third and the forth settings appears in
Figure~\ref{fig:Weibull_10D_boxplot} and
Figure~\ref{fig:Weibull2_10D_boxplot}, respectively. We
compare the risk of standard SVM that ignores censored
observations, censored SVM, censored SVM with variable
selection, and Cox regression. We performed variable
selection for censored SVM based on recursive feature
elimination as in \citet[Section~2.6]{Guyon2002}. When the proportional hazards assumption holds (Setting~3), SVM
performs reasonably well, although the Cox model performs better as expected. When the proportional hazard
assumption fails to hold (Setting~4), SVM performs better
and it seems that the risk of Cox regression converges, but not
to the Bayes risk (see Example~\ref{ex:Cox_Convergence} for discussion). Both
figures show that variable selection achieves a slightly
smaller median risk with the price of higher variance and that
ignoring the censored observations leads to higher risk.

In the fifth setting, we consider a non-smooth conditional
median. We also investigate the influence of using a
misspecified model for the censoring mechanism. The
covariates $Z$ are generated uniformly on the segment
$[-1,1]$. The failure time is normally distributed with
expectation $3+3\indi{Z<0}$ and variance $1$. Note that the
proportional hazards assumption does not hold for the failure time. The censoring variable $C$ follows the Weibull distribution with shape
parameter $2$, and scale parameter $-0.5 Z +log(6)$ which results in mean censoring percentage of $40\%$. Note that for this
model, the censoring is independent of the failure time only
given the covariate $Z$ (see Assumption~(A\ref{as:T_independent_C})). Estimation of the
censoring distribution using the Kaplan-Meier corresponds
to estimation under a misspecified model. Since the censoring
follows the proportional hazards assumption, estimation using the Cox
estimator corresponds to estimation under the true model. We
use $5$-fold-cross-validation to choose the regularization constant and the width of the kernel, as in setting~1.

In Figure~\ref{fig:Jump_fig}, the conditional median
obtained by the censored SVM learning method using both the
misspecified and true model for the censoring, and by Cox
regression, are plotted. The true median is plotted as a
reference. In Figure~\ref{fig:Jump_boxplot}, we compare the
risk of the SVM method using both the misspecified and true
model for the censoring. We also examined the effect of
ignoring the censored observations. Both figures show that
in general SVM does better than the Cox model, regardless
of the censoring estimation. The difference between the
misspecified and true model for the censoring is small and
the corresponding curves in Figure~\ref{fig:Jump_fig}
almost coincide. Figure~\ref{fig:Jump_boxplot} shows
again that there is a non-negligible price for ignoring the
censored observations.

\section{Concluding Remarks}\label{sec:summary}
We studied an SVM framework for right censored data. We
proposed a general censored SVM learning method and showed
that it is well defined and measurable. We derived finite
sample bounds on the deviation from the optimal risk. We
proved risk consistency and computed learning rates. We
discussed misspecification of the censoring model.
Finally, we performed a simulation study to demonstrate the censored SVM method.

We believe that this work illustrates an important approach
for applying support vector machines to right censored
data, and to missing data in general. However, many open
questions remain and many possible generalizations exist.
First, we assumed that censoring is independent of
failure time given the covariates, and the probability that
no censoring occurs is positive given the covariates. It
should be interesting to study the consequences of
violation of one or both assumptions. Second, we have used
the inverse-probability-of-censoring weighting to correct
the bias induced by censoring. In general, this is not
always the most efficient way of handling missing data
\citep[see, for example,][Chapter~25.5]{VDV98}. It would be
worthwhile to investigate whether more efficient methods could
be developed. Third, we
discussed only right-censored data and not general missing
mechanisms. We believe that further development of SVM
techniques that are able to better utilize the data and to
perform under weaker assumptions and in more general settings
is of great interest.

\appendix
\section{Proofs}
\subsection{Auxiliary Results}
The following result is due to~\citetalias{SVR} and is used to prove Theorem~\ref{thm:main}. Since it is not stated as a result there, we state the result and sketch the proof.

\begin{thm}\label{thm:SVR7}
Let $L$ be a loss function and $H$ be an RKHS that satisfies assumptions (B\ref{as:LocallyLipchitzClippable})--(B\ref{as:entropy_bound1}). Let $f_0$ be such that $\|L(z,y,f_0(z))\|_{\infty}\leq B_0$ for some $B_0\geq B$. Fix $\lambda>0$ and $\eta>0$, and let $f\in H$. Then for all $n\geq 72\eta$, with probability not less than $1-e^{\eta}$,
\begin{align*}
&  (P-\ep_n)(L\circ\clip{f}-L\circ\fps)<
\\
&\qquad \frac{17}{27}\left(\lambda\|f\|^2_H+P(L\circ\clip{f}-L\circ\fps)+ \left(\frac{72V\eta}{n}\right)^{\frac{1}{2-\vartheta}}\right)+
  W\left(\frac{a^{2p}}{\lambda^p n}\right)^{\frac1{2-p-\vartheta+\vartheta p}}\,,
\end{align*}
where $W>1$ is a constant that depends only on $p$, $M$, $\vartheta$, and $V$, but not on $f$.
\end{thm}
\begin{proof}
The proof is based on the proofs of Theorems~17.16, 17.20, and~17.23 of \citetalias{SVR}. We now present a sketch of the proof for completeness.

We first note that if $a^{2p}> \lambda^pn$, it follows from~\eqref{eq:sup_bound} that the bound holds for $W\geq 4B$. Thus, we consider the case in which $a^{2p}\leq \lambda^pn$.

Let
\begin{align*}
  r^*=\inf_{f\in H} \lambda\|f\|^2_H+\RLP(\clip{f})-\RLPS\,.
\end{align*}
For every $r> r^*$, write
\begin{align*}
  \mathcal{F}_r=&\left\{f\in H, \lambda\|f\|^2_H+\RLP(\clip{f})-\RLPS\leq r\right\}\,,
  \\
  \mathcal{H}_r=&\left\{L\circ f-L\circ \fps, f\in \mathcal{F}_r\right\}\,.
\end{align*}
Define
\begin{align*}
  g_{f,r}=\frac{Ph_{\clip{f}}-h_{\clip{f}}}{\lambda\|f\|^2_H+\RLP(\clip{f})-\RLPS}\,,\qquad f\in H,\;r>r^*\,.
\end{align*}
Note that for every $f\in H$, $\|g_{f,r}\|_{\infty}\leq 2Br^{-1}$. It can be shown \citepalias[Eq.~7.43 and the discussion there]{SVR} that $Pg_{f,r}^2\leq Vr^{\vartheta-2}$. Using Talagrand's inequality \citepalias[Theorem~7.5]{SVR} we obtain
\begin{align}\label{eq:talagrand_g_fr}
  P\left(\sup_{f\in H}\ep_n g_{f,r}\leq(1+\gamma)P[\sup_{f\in H}| \ep g_{f,r}|]+\sqrt{\frac{2\eta Vr^{\vartheta-2}}{n}}+\left(\frac23+\frac1\gamma\right)\frac{2\eta B}{nr} \right)\geq 1-e^{-\eta}
\end{align}
for every fixed $\gamma>0$. Using Assumption~(A\ref{as:entropy_bound1}), it can be shown that there is a constant $\tilde W$ that depends only on $p$, $M$, $\vartheta$, and $V$, such that for every $r>\tilde W\left(\frac{a^{2p}}{\lambda^p n}\right)^{\frac1{2-p-\vartheta+\vartheta p}} $
\begin{align}\label{eq:complexity_for_g_fr}
  P[\sup_{f\in H}| \ep g_{f,r}|]\leq \frac{8}{30}
\end{align}
\citepalias[see proofs of Theorems~7.20 and~7.23 of][for details]{SVR}.
Substituting $\gamma=1/4$ in~\eqref{eq:talagrand_g_fr}, and using the bound~\eqref{eq:complexity_for_g_fr}, we obtain that with probability of not less than $1-e^{\eta}$,
\begin{align}\label{eq:bound_g_rf}
  \sup_{f\in H}\ep_n g_{f,r}\leq \frac13+\sqrt{\frac{2\eta Vr^{\vartheta-2}}{n}}+\frac{28\eta B}{3nr}
\end{align}
for all $r>\tilde W\left(\frac{a^{2p}}{\lambda^p n}\right)^{\frac1{2-p-\vartheta+\vartheta p}} $.

Using the fact that $n\geq 72\eta$, some algebraic manipulations \citepalias[see][proof of Theorem~7.23 for details]{SVR} yield that for all
$r\geq \left(\frac{72V\eta}{n}\right)^{1/(2-\vartheta)}$
\begin{align}\label{eq:bound_talagrand_g_rf_terms}
  \sqrt{\frac{2\eta Vr^{\vartheta-2}}{n}}\leq \frac16\qquad\,,\qquad \frac{28\eta B}{3nr}\leq \frac{7}{54}\,.
\end{align}
Fix $f\in H$. Using the definition of $g_{f,r}$, together with the estimates in~\eqref{eq:bound_talagrand_g_rf_terms} for the probability bound~\eqref{eq:bound_g_rf}, we obtain that for
\begin{align*}
  r=\tilde W\left(\frac{a^{2p}}{\lambda^p n}\right)^{\frac1{2-p-\vartheta+\vartheta p}}+\left(\frac{72V\eta}{n}\right)^{1/(2-\vartheta)}
\end{align*}
the inequality
\begin{align*}
  (P-\ep_n)(L\circ\clip{f}-L\circ\fps)<\frac{17}{27}\left(\lambda\|f\|^2_H+P(L\circ\clip{f}-L\circ\fps)+r\right)
\end{align*}
holds with probability not less than $1-e^{\eta}$, and the desired result follows.
\end{proof}

\subsection{Proof of Theorem~\ref{thm:main}}\label{sec:thm_main}
\begin{proof}[Proof of Theorem~\ref{thm:main}]
Note that  by the definition of $\fc$,
\begin{align*}
    \lambda\nhs{\fc}+\RLND(\clipfc)\leq     \lambda\nhs{f_0}+\RLND(f_0),
\end{align*}
where $\RLND(f)=\ep_n\delta L(Z,Y(U),f(Z))/\hatG(U|Z)$.
Hence,
\begin{align}\label{eq:boundABCD}
\begin{split}
&\lambda\nhs{\fc}+\RLP(\clipfc)-\RLPS\\
 \leq& \lambda\nhs{f_0}+\RLND(f_0)-\RLND(\clipfc)+\RLP(\clipfc)-\RLPS\\
 =&\left(\lambda\nhs{f_0}+\RLP(f_0) -\RLPS\right)+ \left(\RLND(f_0)-\RLNG(f_0)\right)\\
&\,+ \left(\RLNG(f_0) -\RLP(f_0)+\RLP(\clipfc)-\RLNG(\clipfc)\right)+\left(\RLNG(\clipfc)-\RLND(\clipfc)\right)\\
\equiv & A_n+B_n+C_n+D_n\,,
\end{split}
\end{align}
where
\begin{align*}
\RLNG(f)\equiv \ep_n L_G(Z,U,\delta,f(Z))\equiv
\ep_n\delta L(Z,Y(U),f(Z))/G(T|Z)\,,
\end{align*}
i.e., $\RLNG$ is the
empirical loss function with the true censoring
distribution function.

Using conditional expectation, we
obtain that for every $f\in H$,
\begin{align}\label{eq:conditional_expectation}
\begin{split}
\RLP(f)&\equiv P[L(Z,Y,f(Z))]=
P\left[P\left[\left.\frac{\delta}{G(T|Z)}L(Z,Y,f(Z))\right|Z,T\right]\right]
\\
& =P[L_G(Z,U,\delta,f(Z)]= \RLG(f)\,.
\end{split}
\end{align}
Therefore, we can rewrite the term $C_n$ as
\begin{align}\label{eq:C_n}
\begin{split}
  C_n\equiv& \RLNG(f_0) -\RLP(f_0)+\RLP(\clipfc)-\RLNG(\clipfc)\\
  =& \bigl(\RLNG(f_0) -\RLNG(\fps)\bigr)- \bigl(\RLG(f_0)-\RLG(\fps)\bigr)\\
  &+\bigl(\RLG(\clipfc) -\RLG(\fps)\bigr)- \bigl(\RLNG(\clipfc)-\RLNG(\fps)\bigr)\,,
  \end{split}
\end{align}
where $\fps$ is the Bayes decision function.

For every function $f\in H$, define the functions $h_f:\Z\times\T\mapsto \R$ as
\begin{align*}
  h_f(z,t)=L_G(z,t,f(z))-L_G(z,t,\fps(z))\,,
\end{align*}
for all $z,t\in\Z\times\T$. Using this notation, we can rewrite~\eqref{eq:C_n} as
\begin{align}\label{eq:C_n2}
  C_n\equiv (\ep_n-P) h_{f_0}+(P-\ep_n)h_{\clipfc}\,.
\end{align}
In order to bound $(\ep_n-P) h_{f_0}$ we follow the same arguments that lead to~\citetalias[Eq.~7.42]{SVR}, adapted to our setting. Write
\begin{align*}
  (\ep_n-P) h_{f_0}=(\ep_n-P)( h_{f_0}-h_{\clip{f}_0})+(\ep_n-P) h_{\clip{f}_0}.
\end{align*}
Since $L_G(z,t,f_0(z))-L_G(z,t,\clip{f}_0(z))\geq 0$, we obtain from the definition of $L_G$, \eqref{eq:sup_bound} and the bound on $f_0$ that $ h_{f_0}-h_{\clip{f}_0}\in [0,B_0/2K]$. It thus follows that
\begin{align*}
  P\left(( h_{f_0}-h_{\clip{f}_0})-P( h_{f_0}-h_{\clip{f}_0})\right)^2\leq P( h_{f_0}-h_{\clip{f}_0})^2\leq \frac{B_0}{2K}P( h_{f_0}-h_{\clip{f}_0})\,.
\end{align*}
Using Bernstein's inequality for the function $ h_{f_0}-h_{\clip{f}_0}-P( h_{f_0}-h_{\clip{f}_0})$, we obtain that with probability not less than $1-e^{-\eta}$,
\begin{align*}
  (\ep_n-P)( h_{f_0}-h_{\clip{f}_0})\leq \sqrt{\frac{\eta B_0 P( h_{f_0}-h_{\clip{f}_0})}{Kn}}+\frac{B_0\eta}{3Kn}\,.
\end{align*}
Using $\sqrt{ab}\leq \frac{a}{2}+\frac{b}{2}$, we obtain
\begin{align*}
  \sqrt{\frac{\eta B_0 P( h_{f_0}-h_{\clip{f}_0})}{Kn}}\leq P( h_{f_0}-h_{\clip{f}_0})+\frac{B_0\eta}{4Kn}\,,
\end{align*}
which leads to the bound
\begin{align}\label{eq:bound_h0_1}
  (\ep_n-P)( h_{f_0}-h_{\clip{f}_0})\leq P( h_{f_0}-h_{\clip{f}_0})+\frac{7B_0\eta}{12Kn}\,,
\end{align}
which holds with probability not less than $1-e^{-\eta}$.

Note that by the definition of $L_G$, we have
\begin{align}\label{eq:var_bound_LG}
\begin{split}
  Ph_{\clip{f}}^2&\equiv P\left(\frac{\delta}{G(T-)}(L(Z,Y,\clip{f}(Z))-L(Z,Y,\fps(Z)))\right)\\
  &=P\left(L(Z,Y,\clip{f}(Z))-L(Z,Y,\fps(Z))\right)\\
  &\leq V P\left(L(Z,Y,\clip{f}(Z))-L(Z,Y,\fps(Z))\right)^{\vartheta}=VPh_{\clip{f}}^{\vartheta},
  \end{split}
\end{align}
where we used~\eqref{eq:conditional_expectation} in the equalities and~\eqref{eq:var_bound} in the inequality. Let $\tV=\max\{V,(B/(2K))^{2-\vartheta}\}$.
It follows from the proof of~\citetalias{SVR}, Eq.~7.8, together with~\eqref{eq:var_bound_LG},
that with probability not less than $1-e^{-\eta}$
\begin{align}\label{eq:bound_h0_2}
   (\ep_n-P)h_{\clip{f}_0}\leq Ph_{\clip{f}_0}+ \left(\frac{2\tV \eta}{n}\right)^{\frac{1}{2-\vartheta}}+\frac{2B\eta}{3Kn}\,.
\end{align}
Summarizing, we obtain from~\eqref{eq:bound_h0_1} and~\eqref{eq:bound_h0_2} that
\begin{align}\label{eq:bound_h0_3}
  (\ep_n-P)h_{f_0}\leq Ph_{f_0}+ \left(\frac{2\tV \eta}{n}\right)^{\frac{1}{2-\vartheta}}+\frac{2B\eta}{3Kn}+\frac{7B_0\eta}{12Kn}\,.
\end{align}

We are now ready to bound the second term in~\eqref{eq:C_n2}. By Theorem~\ref{thm:SVR7}, with probability not less than $1-e^{\eta}$,  for all $n\geq 72\eta$,
\begin{align*}
  (P-\ep_n)h_{\clipfc}&<
 \frac{17}{27}\left(\lambda\|\fc\|^2_H+Ph_{\clipfc}+ \left(\frac{72\tV\eta}{n}\right)^{\frac{1}{2-\vartheta}}\right)+
  W\left(\frac{a^{2p}}{\lambda^p n}\right)^{\frac1{2-p-\vartheta+\vartheta p}},
\end{align*}
where $W>1$ is a constant that depends only on $p$, $M$, $\vartheta$, and $\tV$.

We would like to bound the expressions $B_n$ and $D_n$ of \eqref{eq:boundABCD}. Note that for any function $f$, we have
%\begin{align}\label{eq:bound_for_risk_LGD_split}
%\begin{split}
%&|\RLNG(\clip{f})-\RLND(\clip{f})|
%\\
%&\;\leq
%\left|\ep_n\frac{\delta L(Z,Y,\clip{f}(Z))}{G(T|Z)}-\ep_n
%\frac{\delta L(Z,Y,\clip{f}(Z))}{G_P(T|Z)}\right|
%+
%\left|\ep_n\frac{\delta L(Z,Y,\clip{f}(Z))}{G_P(T|Z)}-\ep_n
%\frac{\delta L(Z,Y,\clip{f}(Z))}{\hatG(T|Z)}\right|
%\\
%&\;\leq  \left|\ep_n\frac{\delta L(Z,Y,\clip{f}(Z))}{G(T|Z)\hatG(T|Z)}\left(
%G_P(T|Z)-G(T|Z)\right)\right|
%\\
%&\quad+\ep_n  \left|\frac{\delta L(Z,Y,\clip{f}(Z))}{G_P(T|Z)\hatG(T|Z)}\right|
%\|G_P(T|Z)-\hatG(T|Z)\|_{\infty}
%\\
%&\;\leq \frac{B}{2K^2}\left(|\ep_n(G_P-G)(T|Z)|+\|G_P-\hatG\|_{\infty}\right)\,,
%\end{split}
%\end{align}
\begin{align}\label{eq:bound_for_risk_LGD}
\begin{split}
|\RLNG(\clip{f})-\RLND(\clip{f})|
&\equiv
\left|\ep_n\frac{\delta L(Z,Y,\clip{f}(Z))}{G(T|Z)}-\ep_n
\frac{\delta L(Z,Y,\clip{f}(Z))}{\hatG(T|Z)}\right|
\\
&=  \left|\ep_n\frac{\delta L(Z,Y,\clip{f}(Z))}{G(T|Z)\hatG(T|Z)}\left(
\hatG(T|Z)-G(T|Z)\right)\right|\\
&\leq \frac{B}{2K^2}\ep_n|(\hatG-G)(T|Z)|\,,
\end{split}
\end{align}
where the last inequality follows from
condition~(A\ref{as:positiveRisk}) and~\eqref{eq:sup_bound}.

Summarizing, we obtain that with probability not less than $1-3e^{-\eta}$
\begin{align*}
  \lambda\nhs{\fc}+&\RLP(\clipfc)-\RLPS\\
\leq\;\; & \lambda\nhs{f_0}+ Ph_{f_0}+ \left(\frac{2\tV \eta}{n}\right)^{\frac{1}{2-\vartheta}}+\frac{2B\eta}{3Kn}+\frac{7B_0\eta}{12Kn}+\frac{B}{2K^2}\ep_n|(\hatG-G)(T|Z)|\\
 &+   \frac{17}{27}\left(\lambda\|\fc\|^2_H+Ph_{\clipfc}+ \left(\frac{72\tV\eta}{n}\right)^{\frac{1}{2-\vartheta}}\right)+
  W\left(\frac{a^{2p}}{\lambda^p n}\right)^{\frac1{2-p-\vartheta+\vartheta p}}\,.
\end{align*}
Note that by conditional expectation~\eqref{eq:conditional_expectation}, $RLP(\clipfc)-\RLPS=Ph_{\clipfc}$. Since $\tV>(B/2K)^{2-\vartheta}$ and $n\geq72\eta$,
\begin{align*}
  \frac{2B\eta}{3Kn}\leq \frac{4}{3}\cdot\frac{B}{2K}
  \cdot\frac{1}{72}\cdot\frac{72\eta}{n}\leq \frac{1}{54}\tV^{1/(2-\vartheta)}\left(\frac{72\eta}{n}\right)^{1/(2-\vartheta)}\,.
\end{align*}
Hence, using the fact that $6\leq 36^{1/(2-\vartheta)}$, and some algebraic transformations, we obtain
\begin{align*}
  &\lambda\nhs{\fc}+\RLP(\clipfc)-\RLPS
  \\
  &\leq \frac{27}{10}\left(
   \lambda\nhs{f_0}+Ph_{f_0}+\frac{36^{\frac{1}{2-\vartheta}}}{6}\left(\frac{2\tV \eta}{n}\right)^{\frac{1}{2-\vartheta}}+\frac{1}{54}\left(\frac{72\tV\eta}{n}\right)^{1/(2-\vartheta)}\right.\\
&\left.\quad+ \frac{7B_0\eta}{12Kn}+ 2\cdot\frac{B}{2K^2}\ep_n|(\hatG-G)(T|Z)|+\frac{17}{27}\left(\frac{72\tV\eta}{n}\right)^{\frac{1}{2-\vartheta}}+ W\left(\frac{a^{2p}}{\lambda^p n}\right)^{\frac1{2-p-\vartheta+\vartheta p}}\right)  \\
&\leq \frac{27}{10} ( \lambda\nhs{f_0}+Ph_{f_0})+\frac{27}{10}\frac{22}{27}\left(\frac{72\tV\eta}{n}\right)^{1/(2-\vartheta)}\\
&\quad+ \frac{27}{10}\frac{7B_0\eta}{12Kn}+W\left(\frac{a^{2p}}{\lambda^p n}\right)^{\frac1{2-p-\vartheta+\vartheta p}}+ \frac{27}{10}\frac{B}{K^2}\ep_n|(\hatG-G)(T|Z)|
\\
&\leq 3( \lambda\nhs{f_0}+Ph_{f_0})+3\left(\frac{72\tV\eta}{n}\right)^{1/(2-\vartheta)}
+\frac{8B_0\eta}{5Kn}+\frac{3B}{K^2}\ep_n|(\hatG-G)(T|Z)|+W\left(\frac{a^{2p}}{\lambda^p n}\right)^{\frac1{2-p-\vartheta+\vartheta p}}\,.
\end{align*}

Until now we assumed that $n\geq 72\eta$. Assume now that $n<72\eta$. By substituting the bounds~\eqref{eq:C_n2},~\eqref{eq:bound_h0_3} and~\eqref{eq:bound_for_risk_LGD} in~\eqref{eq:boundABCD}, we obtain the following bound, that holds with probability not less than $1-2e^{-\eta}$, and where we did not use any assumption on the relation between $n$ and $\eta$:
\begin{align*}
  \lambda\nhs{\fc}+\RLP(\clipfc)-\RLPS
 \leq&\left(\lambda\nhs{f_0}+\RLP(f_0) -\RLPS\right) +\frac{B}{K^2}\ep_n|(\hatG-G)(T|Z)|\\
 &+ \left(\frac{2\tV \eta}{n}\right)^{\frac{1}{2-\vartheta}}+\frac{2B\eta}{3Kn}+\frac{7B_0\eta}{12Kn}+(P-\ep_n)h_{\clipfc}\,.
\end{align*}
By the definition of $h_{\clipfc}$, we obtain that $(P-\ep_n)h_{\clipfc}\leq B/K$. Using the fact that $B/2K\leq \tV^{1/(2-\vartheta)}$, we obtain that
\begin{align*}
(P-\ep_n)h_{\clipfc}\leq 2\left(\frac{72\tV\eta}{n}\right)^{1/(2-\vartheta)}  \end{align*}
and thus the result follows also for the case $n<72\eta$.
\end{proof}

\subsection{Additional Proofs}\label{sec:additional_proofs}
\begin{proof}[Proof of Lemma~\ref{lem:clipping}]
Define
\begin{align*}
    g(z,y)=\inf_s L(z,y,s)\,,\qquad\qquad (z,y,s)\in\Z\times\Y\times\R\,.
\end{align*}
By assumption, for every $(z,y)\in\Z\times\Y$, $\inf_s L(z,y,s)$ is obtained at some point $s=s(z,y)$. Moreover, since $L$ is strictly convex, $s_0$ is uniquely defined.

We now show that $g(z,y)$ is continuous at a general point $(z_0,y_0)$. Let $\{(z_n,y_n)\}$ be any sequence that converges to $(z_0,y_0)$. Let $m_n=g(z_n,y_n)=L(z_n,y_n,s_n)$, and assume by contradiction that $m_n$ does not converge to $m_0=g(z_0,y_0)$. Since $g(z,y)$ is bounded from above by $\max_{z,y}L(z,y,0) $ and $\Z\times\Y$ is compact, there is a subsequence $\{m_{n_k}\}$ that converges to some $m^*\neq m_0$. By the continuity of $L$, there is a further subsequence $\{s_{n_{k_l}}\}\in (-\infty,\infty)$ such that $L(z_{n_{k_l}},y_{n_{k_l}},s_{n_{k_l}})= m_{n_{k_l}}$ and $\{s_{n_{k_l}}\}$ converges to $s^*\in [-\infty,\infty]$. If $s^*\in (-\infty,\infty)$, then by definition $m_0=\inf_s L(z_0,y_0,s)< L(z_0,y_0,s^*)=m^*$, and hence from the continuity of $L$ for all $n$ large enough $L(z_{n_{k_l}},y_{n_{k_l}},s_0)< L(z_{n_{k_l}},y_{n_{k_l}},s_{n_{k_l}})$, and we arrive at a contradiction.

Assume now that $s^*\notin(-\infty,\infty)$, and without loss of generality, let $s^*=\infty$. Note that $\max_{z,y}g(z,y)$ is bounded from above by $M_0=\max_{z,y}L(z,y,0) $. Chose $s_M>s_0$ such that for all $s>s_M$, $L(z_0,y_0,s)>3M_0$. By the continuity of $L$, there is an $\eps>0$ such that for all $(z,y)\in B_{\eps}(z_0,y_0)\cap \Z\times\Y$, $L(z,y,s_M)>2M_0$, and note that $L(z,y,s_0)<2M_0$. Recall that $L$ is strictly convex in the last variable, and hence it must be increasing at $s_M$ for all points $(z,y)\in B_{\eps}(z_0,y_0)\cap \Z\times\Y$ \citep[see for example][Proposition~1.3.5]{Niculescu2006Convex}. Consequently, for all $n$ big enough, $ L(z_{n_{k_l}},y_{n_{k_l}},s_{n_{k_l}})>2M_0$, and we again arrive at a contradiction, since $m_{n_{k_l}}<M_0$.

We now show that $s(z,y)=\argmin_s L(z,y,s)$ is continuous at a general point $(z_0,y_0)$. Let $\{(z_n,y_n)\}$ be a sequence that converges to $(z_0,y_0)$. Let $s_n=s(z_n,y_n)$. Assume, by contradiction, that $s_n$ does not converge to $s_0$. Hence, there is a subsequence $\{s_{n_k}\}$ that converges to some $s^*\in(-\infty,\infty)$ ($s^*\in\{-\infty,\infty\}$ cannot happen, see above). Hence, $\lim L(z_{n_k},y_{n_k},s_{n_k})=\lim m_{n_k}=m_0$, and  $L(z_0,y_0,s_0)=L(z_0,y_0,s^*)=m_0$, which contradicts the fact that $L(z_0,y_0,\cdot)$ is strictly convex and therefore has a unique minimizer.

Since $s(z,y)$ is continuous on a compact set, there is an $M$ such that $|s(z,y)|<M$. It then follows from Lemma~2.23 of \citetalias{SVR} that $L$ can be clipped at $M$.
\end{proof}
%
%
%\begin{lem}\label{lem:entropy_clip}
%Let $\F$ be a function space and denote
%$\clip{\F}=\{\clip{f}: f\in\F\}$. Then
%\begin{align*}
%    \mathcal{N}(\clip{\F},\|\cdot\|,\eps)\leq \mathcal{N}(\F,\|\cdot\|,\eps)\,.
%\end{align*}
%\end{lem}
%\begin{proof}
%Assume that $\mathcal{N}(\F,\|\cdot\|,\eps)=K<\infty$,
%otherwise the assertion is trivial. Let $f_1,\ldots,f_K $
%be such that $\F\in \bigcup_{k=1}^K B(f_k,\eps) $, where
%$B(f,\eps)$ is the $\eps$-ball with respect to the norm
%$\|\cdot\|$, centered at $f$. Note that since
%$|\clip{s}-\clip{t}|\leq |s-t|$ for all $s,t\in\R$, we
%have $\|\clip{f_1}-\clip{f_2}\|\leq \|f_1-f_2\|$ and hence
%$\clip{\F}\in \bigcup_{k=1}^K B({\clip{f}}_k,\eps) $. The
%result follows.
%\end{proof}
%-----------------------------------------------------------
\begin{supplement}[id=sec:suppA]
  \sname{Supplement A}
  \stitle{Matlab Code}
\slink[url]{http://stat.haifa.ac.il/~ygoldberg/research}
  \sdescription{Please read
the file README.pdf for details on the files in this
folder.}
\end{supplement}

\bibliographystyle{plainnat}
%\bibliography{censored_svm}
%GATHER{censored_svm.bib}

\end{document}